%% file: main.tex
\documentclass[twoside]{article}

\usepackage[accepted]{aistats2025}

\usepackage{natbib}

\input{utils.tex}

\begin{document}

\runningtitle{Decision from Suboptimal Classifiers: Excess Risk Pre- and Post-Calibration}

\twocolumn[

\aistatstitle{Decision from Suboptimal Classifiers:\\Excess Risk Pre- and Post-Calibration}

\aistatsauthor{Alexandre Perez-Lebel \And Gael Varoquaux \And Sanmi
Koyejo}

\aistatsaddress{
    \begin{tabular}[t]{c}
        Soda, Inria Saclay, France \\[2pt]
        Stanford University, USA\\[2pt]
        Fundamental Technologies$^*$, USA
    \end{tabular}
    \And
        Soda, Inria Saclay, France
    \And
    Stanford University, USA}

\aistatsauthor{Matthieu Doutreligne \And Marine Le Morvan}

\aistatsaddress{Haute Autorité de Santé, France \And Soda, Inria Saclay,
France}

\runningauthor{Alexandre Perez-Lebel, Gael Varoquaux, Sanmi Koyejo,
Matthieu Doutreligne, Marine Le Morvan}

]

\def\thefootnote{*}\footnotetext{Current affiliation}

\begin{abstract}
    Probabilistic classifiers are central for making informed decisions under uncertainty. Based on the maximum expected utility principle, optimal decision rules can be derived using the posterior class probabilities and misclassification costs. Yet, in practice only learned approximations of the oracle posterior probabilities are available. In this work, we quantify the excess risk (a.k.a. regret) incurred using approximate posterior probabilities in batch binary decision-making. We provide analytical expressions for miscalibration-induced regret ($\rcl$), as well as tight and informative upper and lower bounds on the regret of calibrated classifiers ($\rgl$). These expressions allow us to identify regimes where recalibration alone addresses most of the regret, and regimes where the regret is dominated by the grouping loss, which calls for post-training beyond recalibration. Crucially, both $\rcl$ and $\rgl$ can be estimated in practice using a calibration curve and a recent grouping loss estimator. On NLP experiments, we show that these quantities identify when the expected gain of more advanced post-training is worth the operational cost. Finally, we highlight the potential of multicalibration approaches as efficient alternatives to costlier fine-tuning approaches.
\end{abstract}

\section{Introduction}
\label{sec:introduction}

Whether it's marking a financial transaction as fraudulent, or deciding if a suspected cancer warrants a biopsy, making a decision involves carefully weighting the inconvenience of false positives (e.g. an invasive biopsy on a healthy patient) with that of false negatives (e.g. delaying cancer treatment due to a missed diagnosis). Often, the true outcome cannot be deterministically characterized \citep[e.g. in medicine, ][sec 3.1.1]{Sox2013}. A rational decision-maker thus seeks the decision that offers the best harm--benefit tradeoff according to its preferences and the probabilities of each outcome.\looseness=-1

In decision theory \citep{Peterson2017,Kochenderfer2015}, the Maximum Expected Utility Principle (or similarly, the Minimum Expected Cost Principle) offers a framework for optimal decisions. By using the class-conditional probabilities $P(Y|X)$ of the outcome Y given the input data X to quantify the uncertainty, along with the utilities (or costs) associated with decisions, one can derive the decision that maximizes utility. In a binary decision setting, the optimal decision for a record $x$ is 1 whenever the probability of the corresponding outcome $\PPG{Y\!=\!1}{X=x}$ is above a threshold $\t$ function of the utilities \citep{Elkan2001}.

The optimal decision depends on both the utilities and the \emph{oracle} class-conditional probabilities $P(Y|X)$. In practice, the utilities can be defined by an expert, it is a task in itself \citep[chap. 8]{Sox2013}, \citep[sec. 3.1.4-5]{Kochenderfer2015}. However, the oracle probabilities are unknown and must be estimated \eg{} using a learned probabilistic classifier within the machine learning framework. As relying on approximate probabilities affects the resulting decisions, it is essential to choose a model that leads to the best possible decisions.

Models are often selected based on common metrics such as accuracy, AUC or Brier score. Yet, a model's high accuracy is no guarantee of its ability to improve subsequent decisions, and nor are the AUC or Brier score \citep{Localio2012beyond}. Model calibration is also known to be desirable and histogram binning---a well-known recalibration method---was introduced specifically to enhance decision-making by calibrating class-conditional probabilities \citep{zadrozny2001unknowncosts}. \cite{VanCalster2019achilles} highlighted the importance of assessing calibration when using estimated probabilities for clinical decision-making. Nonetheless, it remains unclear what degree of calibration is necessary for a model to be suitable for production or for preferring one model over another.

Validating decision-rules in practice is crucial to ensure that AI does more good than harm in production. For this purpose, decision-analytic measures such as Expected Utility and Net Benefit \citep{vickers2016netbenefit} can be used, with the latter gaining traction in medical communities. In this work, we \emph{investigate how inaccuracies in the estimated class-conditional probabilities translate into regret}, i.e., into an expected utility lower than the best possible expected utility for the given task. The interplay between decision-analytic measures and quantifications of class-conditional inaccuracies has not been thoroughly studied. Most research has focused on mitigating the detrimental effects of miscalibration on resulting decisions \citep{Zhao2021, Rothblum2022}, while \cite{VanCalster2015calibration} evaluated its impact on Net Benefit through simulations.

\paragraph{This work} We study how errors in estimated probabilities affect the optimality of decisions derived from these probabilities. Our theoretical results address practical questions such as: Is a pre-trained model suited to a new task? What is the simplest way to correct decisions based on a sub-optimal probabilistic classifier? How much will a model benefit from post-training?
This approach is particularly valuable in the current trend of applying large pre-trained models to new tasks, rather than training models from scratch---\eg{} using foundation models.
Our contributions are:
\begin{itemize}[itemsep=1pt, parsep=1pt, topsep=0pt,leftmargin=2ex]

    \item We formally describe how discrepancies between estimated probabilities and the underlying distribution $P(Y|X)$ affect the decision regret. In particular, we give an analytical expression for the regret induced by miscalibration, as well as \emph{tight} and \emph{informative} upper and lower bounds on the regret of a re-calibrated classifier. These bounds are distribution-agnostic, model-agnostic, and solely controlled by the grouping loss, decision threshold, and re-calibrated probabilities.

    \item Using these bounds, we describe two regimes: one where recalibration is a cheap and effective post-training strategy and another where recalibration fails to mitigate the regret.

    \item We show that these scenarios can be successfully identified in practice. The bounds identify the cases where fine-tuning improves utility upon recalibration alone on NLP tasks. This enables a new model validation procedure to guide post-training.

    \item We investigate the potential of multicalibration and show it is a cost-effective and controllable alternative to fine-tuning, while concurrently diminishing regret compared to calibration alone.
\end{itemize}

\section{Background}  %
\label{sec:background}

\subsection{Related work}

\paragraph{Cost-sensitive learning} Cost-sensitive learning \citep[chap. 4]{Ling2008cost, Fernandez2018learning} focuses on minimizing expected costs rather than misclassification rates. This can be achieved through three main approaches: direct methods, such as modifying tree-based splitting criteria \citep[sec. 4.4]{Ling2004decision, Petrides2022, Fernandez2018learning} or training losses to embed cost information \citep{Chung2015cost}; pre-processing methods, which alter the training set to account for the cost \citep{Zadrozny2003sampling,Ting1998weighting}; and post-processing approaches, including threshold adjustment and refining class-conditional probabilities. For threshold adjustment, \citet{Sheng2006thresholding} select the threshold that minimizes expected costs on the training set, while recent work seeks to adapt the decision threshold to miscalibrated classifiers \citep{Rothblum2022}.

Unlike direct and pre-processing approaches, post-processing methods do not need model retraining when the costs change. This makes them appealing in a many settings: when the costs are not known at training time, change over time, when the model is too expensive to retrain (\eg{} large pre-trained models), or when the model is not accessible (\eg{} using an API). In this work, we focus on post-training methods in batch decision-making, and in particular on post-training the obtained class-conditional probabilities.

\paragraph{Calibration} Refining class-conditional probabilities often involves recalibration. Calibration ensures that, on average, the predicted probabilities match the positive rate within groups of the same estimated probability. For instance, if a classifier estimates an 80\% probability, then 80\% of those predictions should be actual positive outcomes. Learned classifiers are often miscalibrated; boosted trees tend to be under-confident \citep{Niculescu-Mizil2005b}, whereas naive Bayes classifiers or modern neural networks tend to be over-confident \citep{Guo2017}. To address these issues, many recalibration techniques have been developed, including Platt scaling \citep{Platt1999a}, histogram binning \citep{Zadrozny2001b}, isotonic regression \citep{Zadrozny2002}, or temperature scaling \citep{Guo2017}. Recalibrating  is advocated for better decisions \citep{VanCalster2019achilles}, and some recalibration methods are specifically framed in decision settings, \eg{} multiclass \citep{Zhao2021}.

\paragraph{Beyond calibration: post-training} Calibration, being a control on averages, does not control individual probabilities. A complete characterization of predicted probabilities is given by decomposing the expected loss, and thus prediction errors \citep[3.1 and 5.1]{Kull2015}:
\begin{equation}
\small
    \text{\hspace*{-4ex}\shortstack{Expected\\Loss}} = \underbrace{\text{\shortstack{Calibration \\ Loss}} + \text{\shortstack{Grouping\\Loss}}}_{\text{Epistemic Uncertainty}} + \underbrace{\text{\shortstack{Irreducible\\Loss}}\makebox[0pt]{.}}_{\text{\hspace*{-5ex}Aleatoric Unce\rlap{rtainty}}}
    \label{eqn:decomposition}
\end{equation}
\emph{Aleatoric uncertainty}, stems from the randomness of the outcome $Y$ and cannot be reduced even with an optimal model and infinite data. On the opposite, \emph{epistemic uncertainty} is due to model imperfection and can be reduced with a better model \citep{Hullermeier2021}; it is a good indicator of whether a model can be improved \citep{Lahlou2021}.
The recalibration methods listed above only reduce the calibration loss in eq.\,\eqref{eqn:decomposition}.
Multicalibration recently pushed further, notably for fairness \citep{Hebert-Johnson2018}, as well as calibration within groups \citep{Kleinberg2016,Pfohl2022}. More general post-training methods tackle the full error in eq.\,\eqref{eqn:decomposition}, \eg{} stacking \citep{Pavlyshenko2018}, which learns a model on top of the output of another model, or fine-tuning, particularly useful with the advent of large pretrained models.

\paragraph{Measuring the grouping loss}
Miscalibration is well characterized \citep[as with expected calibration error, ECE,][]{Naeini2015}, however it is only part of the epistemic error.
Recently, \citet{Perez-Lebel2023} gave an estimator for the remainder, the grouping loss, showing that modern classifiers often exhibit grouping loss in real-world settings.
The grouping loss can be seen as the loss of grouping together entities having different odds.
Concretely, it measures the variance of the true individual probabilities within groups of same estimated probabilities.
The estimator uses a partitioning of the feature space to estimate local averages of the true probabilities. This enables detecting dissimilar entities that were grouped together by the model, thus approximating the grouping loss.\looseness=-1

\subsection{Decision-making under uncertainty: definitions and notations}

\paragraph{Decision theory}

In this article, we consider the classic setting of batch supervised machine learning and focus on the binary setting. Let $(X, Y)$ a pair of jointly distributed random variables on $\Xc \times \cbr{0, 1}$. Binary decision rules map each point in $\Xc$ to a decision in $\cbr{0, 1}$. Let $U \in \R^{2 \times 2}$ be a matrix of utilities, where $U_{ij}$ is the utility of predicting $i$ when the true outcome is $Y=j$.\footnote{Utility in the sense of \citet{vonNeumann1944}, which can be viewed as the opposite of a cost.} The expected utility of a decision rule $\delta : \Xc \to \{0, 1\}$ is defined
for $x \in \supp{X}$
as:\footnote{We note $\supp$ the support of a random variable, \ie{} the values for which the probability density function is nonzero.}
\begin{flalign}
    \text{Pointwise} &&&
    \EU(\delta,x) \triangleq \espk{U_{\delta(X), Y}}{X\!=\!x}
    &
    &&
    \label{eq:eu:pointwise}
    \\
    \text{Overall} &&&
    \EU(\delta) \triangleq \esp[X]{\EU(\delta,X)}.
    &&&
    \label{eq:eu}
\end{flalign}
Let $\unorm\!=\!U_{00}\!-\!U_{10}\!+\!U_{11}\!-\!U_{01}$. Decision theory \citep{Elkan2001} states that the best decision rule, in the sense that it maximizes the (pointwise) expected utility,\footnote{Equivalently, minimizes the expected cost.} is:
\begin{equation}
    \label{eq:optimal_classifier}
    \delta^\star : x \mapsto \mathds{1}_{\PPG{Y=1}{X=x}\geq t\sstar}
    \quad \text{where} \quad
    t\sstar \triangleq \tfrac{U_{00}\!-\!U_{10}}{\unorm}.
\end{equation}
The optimal decision thus amounts to assigning class 1 to $x$ whenever $\PPG{Y\!=\!1}{X\!=\!x} \geq t\sstar$ and 0 otherwise. When misclassification costs are equal, \ie{} $U_{00} = U_{11}$, and $U_{10} = U_{01}$, the optimal threshold is $t\sstar = 0.5$, similar to regular cost-insensitive classification. We will denote by $\hstar : x \mapsto \PPG{Y\!=\!1}{X\!=\!x}$ the (unknown) conditional probability of the positive class, and $\hh : \Xc \to [0, 1]$ the probabilistic predictor estimating the probabilities $\hstar$. Without loss of generality, the set of binary decision rules $ \delta : \Xc \to {\{0, 1\}}$
can be parametrized with estimated class-conditional probabilities $\hh : \Xc \to [0, 1]$ and a threshold $t \in [0, 1]$ (\autoref{lem:decision:parametrized}) as:\looseness=-1
\begin{equation}
    \label{eq:decision:rule}
        \delta_{\hh, t} : x \mapsto \1{\hh(x) \geq t}.
\end{equation}
With these notations, the optimal decision rule thus writes $\delta_{f^\star, t^\star}$.
In the following, we consider a candidate decision rule $\delta_{\hh, t}$ with $t \in [0, 1]$.

\paragraph{Calibration and grouping loss}
The recalibrated predictor is defined as $\ch \circ \hh$, where $\ch$ is the calibration curve given by:
\begin{flalign}
    \label{eq:calibration-curve}%
    \text{Calibration curve}
    &&
    \ch: p \mapsto \espk{Y}{\hh(X)\!=\!p}.%
    &&
\end{flalign}
A binary classifier $\hh$ is calibrated when $\ch(p) = p$ for all $p \in \supp{\hh(X)}$.
Finally, writing $\mathbb{V}$ the variance, the grouping loss associated to the squared loss is defined as:
\begin{flalign}
    \text{Grouping loss}
    &&
    \GL : p \mapsto \vfhp.
    &&
\end{flalign}
It can be thought of as the variance of the unknown individual probabilities around their mean $\ch(p)$ in the bin $p$ of a reliability diagram \citep{Perez-Lebel2023}.

\section{Theory: regret on a decision}

\label{sec:theory}
In this section, we establish a connection between the errors in the estimated probabilities (a form of uncertainty quantification) and the suboptimality (or regret) of the resulting decisions.

\subsection{Regret decomposition of suboptimal decision rules}

While the best decision rule is given by $\delta_{\hstar, t\sstar}$, only imperfect estimates $\hh$ of $\hstar$ are available in practice.
It is thus of interest to characterize the optimal decision rule that can be achieved from $\hh$.
\begin{restatable}[Best decision given estimated probabilities, \ref{sec:proof:best-decision:dh}]{proposition}{thmBestDh}
\label{prop:decision:est}
    Let $\mathcal{D}_{\hh}$ be the set of decision rules \underline{function of the estimated probabilities $\hh$}. Then the calibrated probabilities thresholded at $t^\star$,
    \begin{equation}
        \delta_{\cch, t^\star}: x \mapsto \1{(\cch)(x) \geq t^\star},
    \end{equation}
    maximize the conditional expected utility over $\mathcal{D}_{\hh}$, i.e.,
    \begin{equation*}
        \delta_{\cch, t^\star} \in \argmax{\delta \in \mathcal{D}_{\hh}} \EU(\delta|p) \qquad
        \text{for all } p \in \supp{\hh(X)}
    \end{equation*}
    \vspace{-4mm}
    \begin{align*}
        &\text{with} \quad \EU(\delta|p) \triangleq \espk{\EU(\delta,X)}{f(X) = p}\\
        &\text{and} \quad \mathcal{D}_{\hh} = \cbr{\delta_{g\circ\!\hh\!, t} : \quad g: [0, 1] \to [0, 1],\; t \in [0, 1]}.
    \end{align*}
\end{restatable}
\autoref{prop:decision:est} shows that given an approximation $\hh$ of the oracle probabilities $\hstar$, recalibrating the classifier and using $t\sstar$ as threshold achieves the highest utility among decisions based on $\hh$ only.
While it is commonly admitted that recalibration is desirable to improve decisions, to the best of our knowledge, the optimality of $\delta_{\cch, \t}$ in batch binary decision-making has not been previously demonstrated.
Note that $\mathcal{D}_{\hh}$ contains all possible decision rules taking only $\hh$ as input (\autoref{lem:dh}). This includes recalibration but excludes decision rules based on both $\hh(x)$ and $x$. This prevents adjusting for the grouping loss, thereby incurring regret relative to $\delta_{\hstar\!, t\sstar}$.\looseness=-1

Since $\delta_{\cch, t\sstar}$ is the best decision over a subset $\mathcal{D}_{\hh}$ of all possible decision rules, it holds for $p \in \supp{\hh(X)}$ that:
    \begin{equation}
        \label{eq:utility:ranks}
        \underbrace{\EU(\delta_{\hh, t}|p)}_{\text{Naive}} \;\; \leq \;\; \underbrace{\EU(\delta_{\cch, t\sstar}|p)}_{\text{Recalibrated}} \;\; \leq \;\; \underbrace{\EU(\delta_{\hstar, t\sstar}|p)}_{\text{Oracle}}
        .
    \end{equation}
Ranking \eqref{eq:utility:ranks} leads us to define the conditional \emph{calibration regret} $\rclht(p)$ as the expected utility gap between the best decision given $f$ and the candidate decision: $\rclht(p) \triangleq
\EU(\delta_{\cch, t\sstar}|p) - \EU(\delta_{\hh, t}|p)$. $\rclht$ quantifies the regret of using a miscalibrated classifier rather than a calibrated one.
Similarly, we define the conditional \emph{grouping regret} $\rglh(p)$ as the expected utility gap between the oracle decision and the best decision given $f$: $\rglh(p) \triangleq \EU(\delta_{\hstar, t\sstar}|p) - \EU(\delta_{\cch, t\sstar}|p)$.
$\rglh$ quantifies the regret of using the recalibrated classifier $\cch$ instead of oracle class-conditional probabilities $\hstar$.
The conditional \emph{regret} $\rht(p)$ of the candidate decision to the oracle decision $\deltastar$
is defined as $\rht(p) \triangleq \EU(\delta_{\hstar, t\sstar}|p) - \EU(\delta_{\hh, t}|p)$
and can naturally be decomposed as:
\begin{equation}
    \label{eq:regret-decomposition}
    \rht(p) = \underbrace{\rclht(p)}_{\geq 0} + \underbrace{\rglh(p)}_{\geq 0}.
\end{equation}

The overall regrets can then be obtained by integrating over $p$. Ranking~\eqref{eq:utility:ranks} and decomposition~\eqref{eq:regret-decomposition} trivially hold marginally. We note $\rht$, $\rclht$ and $\rglh$ the marginal counterparts.
Importantly, the literature often focuses on the calibration regret $\rclht$ \citep[\eg{},][\say{type regret}]{Zhao2021,Noarov2023}, which is blind to the grouping regret. In this work, we consider the full regret $\rht$ between the candidate and oracle decisions.
In general, $\rglh$ is nonzero due to the grouping loss of the estimated probabilities $\hh$.
In the following sections \ref{sec:rcl}--\ref{sec:rgl} we provide an analytical expression for the calibration regret $\rclht(p)$ as well as bounds on the grouping loss regret $\rglh(p)$.

\subsection{Expression of the regret stemming from miscalibration}
\label{sec:rcl}
The calibration regret can be estimated using the calibration curve as described in \autoref{prop:rcl:expression}.
\begin{restatable}[Expression of the calibration regret, \ref{sec:proof:rcl:expression}]{proposition}{propRCLExpression}
    \label{prop:rcl:expression}
    For all $p \in \supp{\hh(X)}$,
    \begin{align}
        \rclht(p) & = \begin{cases}
            \unorm |\ch(p) - t\sstar| & \text{if }\; \1{\ch(p) \geq t\sstar} \neq \1{p \geq t}\\
            0 & \text{otherwise}
        \end{cases}
        .
        \label{eq:rcl}
    \end{align}
\end{restatable}
\begin{figure}[t!]
    \centering
    \includegraphics[width=\linewidth]{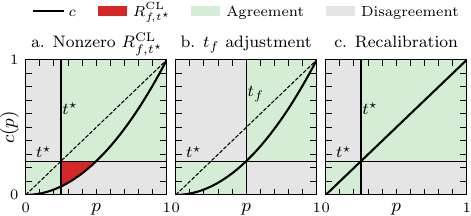}
    \caption{\textbf{Impact of miscalibration on the regret $\rclht$.} (a) The oracle decision $p \mapsto \1{p \geq t\sstar}$ applied on miscalibrated estimated probabilities $\hh$, that is $\delta_{\hh, t\sstar}$, yields a non zero regret $\rclhtstar$ within areas of disagreement with amount $|\ch - t\sstar|$ (red area). The regret $\rclht$ can be reduced to 0 either by adapting the decision to a new threshold $t_\hh = \ch^{-1}(t\sstar)$, that is $\delta_{\hh, t_\hh}$ (b), or by recalibrating the estimated probabilities and using $t\sstar$ as threshold, that is $\delta_{\cch, t\sstar}$ (c) (\autoref{prop:th}).}
    \label{fig:rcl:source}
\end{figure}

$\rclht$ scales as the distance between $t\sstar$ and the calibration curve $\ch$ in areas of disagreements between $\delta_{\hh, t}$ and $\delta_{\cch, t\sstar}$ (red area in \autoreff[a]{fig:rcl:source}). Whenever $\delta_{\hh, t}$ and $\delta_{\cch, t\sstar}$ agree everywhere, $\rclht = 0$.
This is in particular the case for calibrated probabilities with threshold $t\sstar$ (\autoreff[c]{fig:rcl:source}). Yet, it is not necessary for $\hh$ to be calibrated to imply $\rclht = 0$, as the decision from a miscalibrated classifier $\delta_{\hh, t}$ can agree everywhere with $\delta_{\cch, t\sstar}$ (\autoreff[b]{fig:rcl:source}). \autoref{prop:th} shows how to achieve this by threshold adjustment.

\begin{restatable}[Adjusting the threshold $t_\hh$, \ref{sec:proof:th}]{proposition}{thmTH}
\label{prop:th}
    For all $t\sstar \in [0, 1]$
    let $t_\hh \in \ch^{-1}(\{t\sstar\})$ if it exists, otherwise
    let $t_\hh \triangleq \inf \{t: \ch(t) \geq t\sstar \}$.
    If $\ch$ is monotonic non-decreasing,
    then thresholding $\hh$ at $t_\hh$, \ie{} $\delta_{\hh,t_\hh}$,
    achieves zero miscalibration regret:
    $\rcl_{\hh,t_\hh} = 0$.
\end{restatable}

When the calibration curve is non-decreasing (which is assumed by isotonic regression), \autoref{prop:th} shows that instead of recalibrating the classifier $\hh$, one can achieve zero miscalibration regret by adjusting the threshold: $\delta_{\cch, t^\star}\!=\!\delta_{\hh, c^{-1}(t^\star)}$ (\autoreff[b]{fig:rcl:source}).
Note that these two approaches are equally costly as calibrated probabilities must be estimated in both cases. Calibration however has the advantage of being usable even if the calibration curve is not monotonic.
The dual formulation between optimal threshold on $\hh$ and calibration brings another insight: if the calibrated probabilities still hide over- or underconfident subgroups (i.e., nonzero grouping loss), it could lead to suboptimal decisions that cannot be tackled simply by setting a better threshold.

\subsection{Bounding the grouping-loss-induced regret of the calibrated classifier}
\label{sec:rgl}
\vspace{-2mm}
Calibration does not directly control individual probabilities $\f$: within a bin $p$, the individual probabilities $f^\star(x)$ may vary around their mean $c(p)$ with a variance $\GL(p)$. In what follows, we address the regret arising from this variance. Unlike miscalibration, there is no one-to-one correspondence between a specific grouping loss level and the resulting regret $\rglh$. Specifically, $\GL(p)$ represents the variance of individual probabilities, while $\rglh$ depends on their distribution. As shown by the expression of the $\GL$-induced regret (\autoref{thm:rgl:expression}), this regret depends on the proximity of individual probabilities $f^\star(x)$ to the threshold $t^\star$, as well as on the region of agreement between $\delta_{f^\star, t^\star}$ and $\delta_{c(p), t^\star}$. Hence for the same $\GL$, the regret $\rglh$ can vary. This is why we provide lower and upper bounds on $\rglh$. Our bounds are derived by identifying, among all distributions $\PPG{\hstar(X)}{\hh(X)=p}$ with a fixed mean $c(p)$ and variance $\GL(p)$, the distributions that result in the lowest and highest regrets.\looseness=-1

\begin{restatable}[Grouping regret lower bound, \ref{proof:regret:lb}]{theorem}{thmRegretLB}
\label{thm:regret:lb}
    The conditional grouping regret is lower bounded for all $p \in \supp{\hh(X)}$ as $ \rglh(p) \geq \lglh(p)$, by:
    \begin{equation}
        \lglh(p) \triangleq \unorm\left[\GL(p)-\vmin(p)\right]_{+}
    \end{equation}
    \vspace{-7mm}
    \begin{align*}
        &\text{with:} [\cdot]_+ = \max \{\cdot, 0\}\\
        &\text{and:}\;\; %
        \vmin(p)\triangleq \begin{cases}(1\!-\!\ch(p))\left(\ch(p)\!-\!t^{\star}\right) & \!\!\text { if } \ch(p) \geq t^{\star} \\ \ch(p)\left(t^{\star}\!-\!\ch(p)\right) & \!\!\text { otherwise } %
        \end{cases}
        .
    \end{align*}
    \textbf{Tightness.} The lower bound is tight. %
    For any $p \in \supp{\hh(X)}$ for which $\hh$ admits at least 3 antecedent values,
    and for all admissible mean $\ch(p)\in[0, 1]$ and variance $\GL(p) \in [0, \ch(p)(1-\ch(p))]$, there exists a distribution of $(X, Y)$ such that the conditional distribution $\PPG{\hstar(X)}{\hh(X)=p}$ has mean $\ch(p)$ and variance $\GL(p)$, and the grouping regret attains its lower bound: $\rglh(p) = \lglh(p)$.

\end{restatable}

Here $\vmin(p)$ represents the largest possible variance GL with zero regret.
This is achieved by a distribution of the true probabilities $\PPG{\hstar(X)}{\hh(X)=p}$ having all its mass on 0 and $t^\star$, or $t^\star$ and 1 depending on the side of $c(p)$ to $t^\star$.
Such a distribution ensures that $\delta_{\f, \t}$ and $\delta_{c(p), \t}$ always agree, as individual probabilities always lie on the same side of the threshold as $c(p)$, thus guaranteeing zero regret.

\begin{figure}[b!]
    \centerline{%
	\includegraphics[width=.8\linewidth]{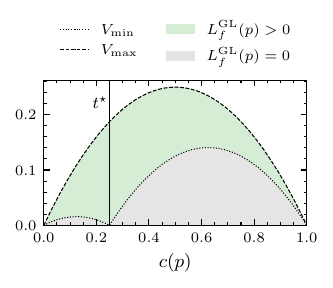}%
    }%
    \caption{\textbf{Impact of the grouping loss on the minimal regret $\lglh$ of a recalibrated classifier.} In a bin $p \in \supp{\hh(X)}$, the grouping loss exceeding $\vmin(p)$ incurs to the recalibrated classifier a nonzero regret $\rglh(p)$ of at least $\unorm \left[\GL(p)-\vmin(p)\right]_{+}$ (\autoref{thm:regret:lb}). Measuring a variance smaller than $\vmin(p)$ is not informative with respect to the grouping regret as there exists $(\hstar, \hh)$ where $\GL(p) = \vmin(p)$ and $\rglh(p) = 0$. The variance cannot exceed $\vmax(p) \triangleq \ch(p)(1 - \ch(p))$. The informative area is highlighted in green.%
    \label{fig:rgl:lb}%
    }
\end{figure}

Measuring a GL in the range $[0, \vmin]$ does not provide informative insights on the regret (\autoref{fig:rgl:lb}), as it could correspond to a distribution with zero regret, like the one described above. Yet it is an ``adversarial'' distribution in the sense that most distributions with lower variance are still likely to incur some regret. Conversely, all distributions with a $\GL$ higher than $\vmin$ will necessarily incur regret proportional to $\GL(p) - \vmin(p)$: their high variance implies that some $\f$ values exist on the wrong side of the threshold, implying disagreements between $\delta_{\f, \t}$ and $\delta_{c(p), \t}$, and thus regret.

Interestingly, $\vmin$ equals 0 for $\ch = t^\star$ and is small for values close to $t^\star$ (\autoref{fig:rgl:lb}). On the threshold, the presence of grouping loss necessarily incurs some regret. This highlights that the more the grouping loss occurs when the calibrated probability $\ch$ is close to the decision threshold $t\sstar$, the more likely it will incur regret. The grouping loss matters the most in regions of high uncertainty, \eg, when the calibrated probabilities $\ch$ are close to 0.5 with a threshold at $\t = 0.5$.

\begin{restatable}[Grouping regret upper bound, \ref{proof:regret:ub}]{theorem}{thmRegretUB}
\label{thm:regret:ub}
    The conditional grouping regret is upper bounded for all $p \in \supp{\hh(X)}$ as $\rglh(p) \leq \uglh(p)$, by:\looseness=-1
    \begin{equation}
        \uglh\!(p) \triangleq
         \tfrac{1}{2}\unorm\!\!\pr{\!\sqrt{\GL(p)\!+\!(\c(p)\!-\!t^{\star})^2}\!-\!|\c(p)\!-\!t^{\star}|\!}
        .
    \end{equation}
    \textbf{Tightness.}
    The upper bound is tight when $t\sstar = \frac{1}{2}$. For any $p \in \supp{\hh(X)}$,
    and for all admissible mean $\ch(p)\in[0, 1]$ and variance $\GL(p) \in [0, \ch(p)(1-\ch(p))]$,
    there exists a distribution of $(X, Y)$ such that the conditional distribution $\PPG{\hstar(X)}{\hh(X)=p}$ has mean $\ch(p)$ and variance $\GL(p)$, and the grouping regret attains its upper bound: $\rglh(p) = \uglh(p)$.
\end{restatable}
This upper bound was obtained by finding the distribution $\PPG{\hstar(X)}{\hh(X)=p}$, with fixed mean $c(p)$ and variance $\GL(p)$, that leads to the largest regret (details in the proof \ref{proof:regret:ub}). A more compact upper-bound immediately follows since $\uglh(p) \leq \tfrac{1}{2}\unorm\sqrt{\GL(p)}$, with equality when $\ch(p) = t\sstar$. This upper bound scales as a squared root of the grouping loss, meaning that large $\GL$ opens the door to large $\rglh$, and small $\GL$ implies small $\GL$-induced regret. In particular $\GL(p) = 0$ implies $\rglh(p) = 0$.

\paragraph{$\lglh$ and $\uglh$ are informative.}
The upper and lower bounds are entirely defined by the calibration curve $\ch(p)$, the grouping loss $\GL(p)$, as well as the threshold $\t$. \autoref{fig:lb:ub} plots the bounds for different values of these quantities. Importantly, it shows that the bounds are informative as the gap between them is not too large. It also illustrates that the $\GL$-induced regret is larger when $\ch$ is close to the decision threshold, and that it increases when $\GL$ increases.

We underline that these bounds are \emph{distribution-agnostic} and \emph{model-agnostic}, requiring no assumptions on either the data distribution, $\f$, or the probabilistic classifier $\hh$. Moreover, thanks to the $\GL$ estimator proposed by \citet{Perez-Lebel2023}, all quantities involved in these bounds can be evaluated in practice.

\begin{figure}[t!]
    \centering
    \includegraphics[width=\linewidth]{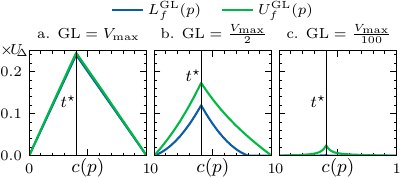}
    \caption{\textbf{Impact of the grouping loss on the bounds of the regret of the recalibrated classifier.}
    Lower and upper bounds $\lglh(p)$ and $\uglh(p)$ as a function of the calibrated probabilities $\ch(p)\in[0, 1]$ for a bin $p \in \supp{\hh(X)}$, in three settings of grouping loss: maximal (a), intermediate (b) and small (c).
    The gap between the lower and upper bounds reduces when the grouping loss is high or low.
    $\vmax \triangleq \ch(p)(1 - \ch(p))$.
    }
    \label{fig:lb:ub}
\end{figure}

\begin{definition}[Regret estimators]
    \label{def:intro:estimators}
    Let $\chat : [0, 1] \to [0, 1]$ and $\GLhat : [0, 1] \to [0, \tfrac{1}{4}]$ be the estimates of $\c$ and $\GL$. We define the plugin estimators of the conditional regrets and bounds for $p \in [0, 1]$ as:
    \vspace{-2mm}
    \begin{align*}
        \lglhhat(p) & \triangleq \textstyle\unorm\left[\GLhat(p)-\vmin(p)\right]_{+}\\
        \uglhhat(p) & \triangleq \textstyle\tfrac{1}{2}\unorm\!\!\pr{\!\sqrt{\GLhat(p)\!+\!(\chat(p)\!-\!t^{\star})^2}\!-\!|\chat(p)\!-\!t^{\star}|\!}\\
        \rglhhat(p) & \triangleq \tfrac{1}{2}(\lglhhat(p) + \uglhhat(p))\\
        \rclhthat(p) & \triangleq \unorm |\chhat(p) - t\sstar| \1{\1{\chhat(p) \geq t\sstar} \neq \1{p \geq t}}\\
        \rhthat(p) & \triangleq \rclhthat(p) + \rglhhat(p)
        .
    \end{align*}
\end{definition}
We note $\lglhhat$, $\uglhhat$, $\rglhhat$, $\rclhthat$, and $\rhthat$ their counterpart obtained by averaging over the bins of $\hh(X)$.

\paragraph{Two regimes} Comparing the conditional calibration and grouping regrets highlights two different regimes. Within a bin $p$, when $\ch(p)$ is close to $t^{\star}$, grouping loss matters more than miscalibration in terms of regret. When $\ch(p) = t\sstar$, the effect of the grouping loss on the regret is maximal: $\GL(p) \leq \rglh(p) \leq \tfrac{1}{2}\sqrt{\GL(p)}$ (\autoref{fig:lb:ub}). Conversely when $\ch(p)$ is further from $t^{\star}$, miscalibration leading to disagreement between $\delta_{\hh,t}$ and $\delta_{\ch\circ\hh,t\sstar}$, if any, typically matters more (\autoref{eq:rcl}).
On the overall population, when averaging across bins, these effects blend according to the distribution of $\hh(X)$.
Conditional estimators (\autoref{def:intro:estimators}) enables detecting these regimes.

\subsection{Grouping Loss Adaptative Recalibration}
\label{sec:glar}

Recalibration only addresses part of the overall regret $\rht$, leaving unchanged the grouping-loss part.
The estimation of $\GL(p)$ from \cite{Perez-Lebel2023} involves finding regions in the input space $\Xc$ that explain the variance of the true probabilities $\hstar(X)$ whithin a bin $\hh(X)=p$.
This work naturally motivates a multicalibration method to reduce grouping loss by using the same estimated region probabilities.

\begin{definition}[GLAR]
    \label{def:glar}
    Let $\mathcal{P} : \Xc \to \R$ be a partition of the feature space $\Xc$.
    We define the Grouping Loss Adaptative Recalibration (GLAR) of a function $\hh : \Xc \to \R$ relative to a partition $\Pc$ as:
    \begin{equation}
        \label{eq:recal_ours}
        \begin{aligned}
        \hhp \colon & \Xc \to [0, 1] \\
        &x \mapsto \espk{Y}{\hh(X) = \hh(x),\mathcal{P}(X)=\mathcal{P}(x)}
        .
        \end{aligned}
    \end{equation}
\end{definition}
The GLAR correction consists of replacing the output of $f$ with that of $f_{\Pc}$.
GLAR provides a calibrated classifier $f_{\Pc}$ that has a lower grouping loss than the original classifier (\autoref{prop:glar}).
Moreover, the decision based on the GLAR-corrected estimator $\hh_{\Pc}$, \ie{} $\delta_{\hhp, t\sstar}$, yields a better expected utility than both the decision based on the initial classifier $\delta_{\hh,t}$ and the recalibrated classifier $\delta_{\cch, t\sstar}$ (\autoref{prop:hierarchy}).
The choice of partition $\Pc$ is important. A trivial partition in one region would give histogram binning recalibration. A too fined-grained partition would give bad estimations because of a low number of samples per partition.
GLAR has the notable advantage of reusing the partitions and local probabilities computed for the estimation of $\rglhhat$, and thus does not incur any additional cost once the regret is estimated.
The implementation details of the method are given in \autoref{sec:details}.

\section{Experiments: validation of regret bounds and link with post-training}
\label{sec:experiments}

\paragraph{Settings} For evaluation on practical scenarios, we consider a hate-speech detection task on real-world language datasets. From the perspective of a platform (\eg{} a social network), failing to identify hate speech will incur a reputation cost, while wrongly identifying a text as hate speech will incur an opportunity cost (\eg{} loss of content and users in the long term). Formally, these costs can be gathered into a $2 \times 2$ utility matrix $U$ to be determined by the platform. The goal is then to solve a binary cost-sensitive decision-making problem with classes \textit{hate speech} ($Y = 1$) and \textit{no hate speech} ($Y=0$).
We investigate post-training of pre-trained models on hate-speech datasets, with potential distribution shift. We benchmark 6 pre-trained models (\autoref{tab:model_translation}) on 14 real-world datasets (\autoref{tab:ds_translation}), and 9 post-training methods.
These include calibration methods, both classical (isotonic regression \citep{Zadrozny2002}, Platt scaling \citep{Platt1999a}, histogram binning \citep{Zadrozny2001b}) and more recent (Scaling-Binning \citep{Kumar2019}, Meta-Cal \citep{Ma2021}); finetuning, where only the last layer is fine-tuned; stacking, where \texttt{scikit-learn}'s \citep{Pedregosa2011} Random Forests or Gradient Boosted Trees are trained on the concatenation of the inputs with the probabilistic outputs of the pretrained model; and finally the GLAR correction (\autoref{sec:glar}).

For each model, we extract the embedding space (usually the penultimate layer) and consider post-training from this representation to the class probability space. We draw utility matrices corresponding to 11 values of $t\sstar$ in the range $[0.01, 0.99]$. We consider decision rules formed by thresholding estimated probabilities at $t\sstar$, \ie{} $\delta_{\hh, t\sstar}$.
Experimental details are given in \autoref{sec:details}.
Our experimental question is whether our regret estimators (\autoref{def:intro:estimators}) are better than classical performance metrics (Brier score, AUC, accuracy, calibration errors defined in \autoref{sec:details}) at identifying post-training gains.

\begin{figure}[b!]
    \centering
    \raisebox{2mm}{%
    \includegraphics[width=0.48\linewidth]{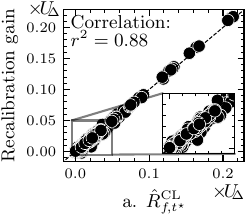}}%
    \includegraphics[width=0.52\linewidth]{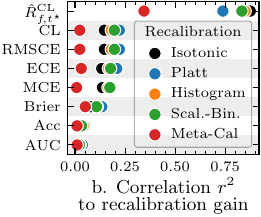}
    \caption{\textbf{$\rclhthat$ captures the gain of recalibration.} (a) Gain in utility of isotonic recalibration versus the regret to the recalibrated classifier $\rclhtstarhat$, for each (model, dataset, $t\sstar$). (b) Pearson's $r^2$ correlation of the gain in utility of each recalibration method with $\rclhtstarhat$ and other metrics.}
    \label{fig:rcl}
\end{figure}

\paragraph{$\rclhtstar$ captures the gain of recalibration} First, we estimate the calibration regret $\rclhtstar$ of each model $\hh$ using the estimator $\rclhtstarhat$ from \autoref{def:intro:estimators}. We estimate $\ch$ using an histogram binning with 15 equal-mass bins.
We compare the expected gain of recalibration, $\rclhtstarhat$, to the gain obtained by recalibrating the pre-trained models using Isotonic Regression. \autoreff[a]{fig:rcl} shows a near perfect identity relation between the estimated calibration regret $\rclhtstar$, and the gain obtained by recalibrating the model (Pearson's correlation: $r^2 = 0.88$). \autoreff[b]{fig:rcl} shows the correlation between $\rclhtstar$ and the gain obtained with 4 other recalibration methods: Platt Scaling, Histogram Binning, Scaling-Binning, and Meta-Cal.

Miscalibrated models do not necessarily incur regret compared to the calibrated classifier. Indeed, the 4 calibration error metrics (ECE, MCE, RMSCE, and CL) are very poorly correlated to the gain of recalibration for any of the 5 recalibration approaches (\autoreff[b]{fig:rcl}, see \autoref{sec:app:gains} for detailed figures). This is because for a given utility level $t\sstar$, the model does not need to be calibrated to achieve $\rcl = 0$ as shown in \autoref{sec:rcl}.
The expected gain from recalibration is best measured by $\rclhthat$.\looseness=-1

\begin{figure}[t!]
    \centering
    \raisebox{4mm}{%
    \includegraphics[width=0.48\linewidth]{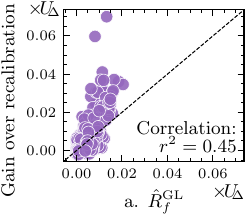}}%
    \includegraphics[width=0.52\linewidth]{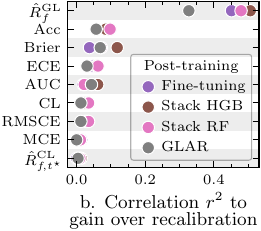}
    \caption{\textbf{Gain of Post-Training on top of recalibration.} (a) Gain in utility of fine-tuning over isotonic recalibration versus $\rglhhat$, for each (model, dataset, $t\sstar$). (b) Pearson's $r^2$ correlation of the gain in utility over isotonic recalibration with $\rglhhat$, $\rclhtstarhat$, and other metrics.\looseness=-1}
    \label{fig:post:excess}
\end{figure}

\paragraph{Gain of post-training on top of recalibration}
The calibrated classifier can be suboptimal due to the grouping loss and $\rglh$. We demonstrate this with 4 post-training methods (fine-tuning, stacking with boosted trees or random forest, and GLAR). We compare their gain in utility to the gain of isotonic recalibration (\autoreff[a]{fig:post:excess}). $\rglhhat$ is by far the metric that best explains the excess gain of post-training among all the other metrics. Across the 4 post-training methods, $\rglhhat$ has $r^2 \approx 0.5$ while calibration errors metrics (ECE, MCE, RMSCE, CL) or model-performance metrics (Brier, AUC) all have $r^2 \leq 0.1$ (\autoreff[b]{fig:post:excess}). The lower levels of correlation of $\rglhhat$ compared to the levels obtained for the calibration regret $\rclhthat$ in \autoref{fig:rcl} are partly due to the fact that $\rglhhat$ is derived from bounds on the regret. For any value of $\rglh$, the regret can vary in a range, hence decreasing the correlation. The fact that $y \gtrsim x$ on \autoreff[a]{fig:post:excess} means that post-training yields a higher gain than what was given by $\rglhhat$. This is expected since the grouping loss estimator can only capture a fraction of the total $\GL$.

\begin{figure}[t!]
    \centering
    \hfill%
    \includegraphics[width=0.95\linewidth]{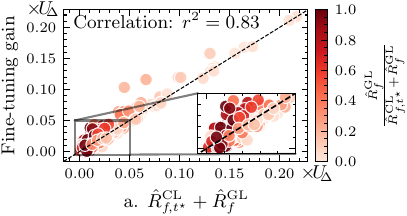}%
    \hspace*{-1mm}

    \vspace{5mm}
    \includegraphics[width=0.9\linewidth]{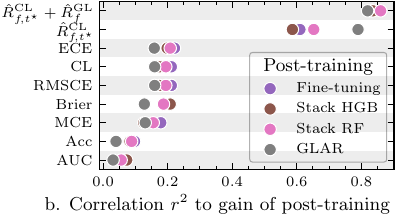}
    \hspace{7mm}
    \caption{\textbf{Gain of Post-Training.} (a) Gain in utility of fine-tuning versus $\rhtstarhat = \rclhtstarhat + \rglhhat$, for each (model, dataset, $t\sstar$). (b) Pearson's $r^2$ correlation of the gain in utility of each post-training method with $\rglhat$, $\rclhthat$, and other performance metrics.}
    \label{fig:post}
\end{figure}

\begin{figure*}[t!]
    \centering
    \includegraphics[width=\linewidth]{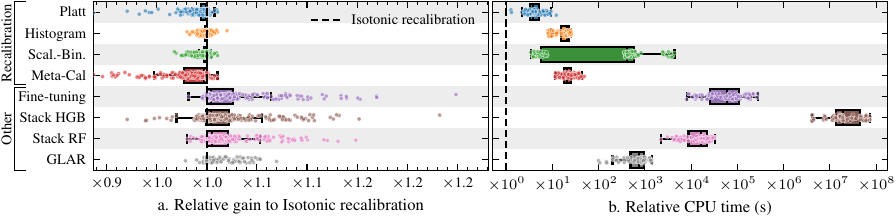}
    \caption{\textbf{Comparing various recalibration and advanced post-training.} Gain (a.) or CPU time (b.) of post-training methods relative to isotonic recalibration, for each (model, dataset, $t\sstar$).}
    \label{fig:perf:iso}
\end{figure*}

\paragraph{Gain of post-training}
Theory points to using $\rhtstarhat = \rclhtstarhat + \rglhhat$ to measure the potential utility gain from using a good post-training method. \autoreff[a]{fig:post} shows that the estimated $\rhtstarhat$ captures well the gain of fine-tuning ($r^2 = 0.83$ and a relation $y \approx x$). \autoreff[b]{fig:post} shows that the gain of stacking (with boosted trees or random forest) and GLAR reaches similar levels of correlation. $\rclhtstarhat$ reaches a lower correlation of $r^2 = 0.6$ which is expected since $\rclht$ is blind to gains on top of recalibration. All other metrics, calibration errors, Brier score and AUC correlate poorly to the gain of post-training ($r^2 \leq 0.2$).

\paragraph{Improving recalibration hits diminishing returns}
In the binary setting, isotonic regression is one of the simplest recalibration methods and is parameter-free. \autoreff[a]{fig:perf:iso} shows the utility gain of each of the remaining 8 post-training methods relatively to the gain of isotonic regression. Isotonic regression provides better utility gains that any of the other recalibration methods, even the more advanced Scaling-Binning and MetaCal (see also \autoref{fig:recal:excess:t}). \autoreff[b]{fig:perf:iso} shows the computational time of the post-training methods. Isotonic regression is also the fastest recalibration method, by at least a factor 10. Fine-tuning and stacking provides better gain than recalibration, but their cost can become prohibitive: stacked boosted trees are $10^7$ times more expensive than isotonic regression. GLAR provides a cheaper post-training method but with lower gains.
These observations along with theoretical results on the regret, suggest that recalibration reaches a ceiling in its ability to improve decision-making. Enhancing recalibration methods often does not lead to better decision-making. Instead, the focus should be on methods to reduce the grouping regret, such as stacking, fine-tuning or developping cheaper alternatives in the spirit of GLAR.\looseness=-1

\section{Conclusion}

\label{sec:conclusion}

This work quantifies both theoretically and empirically how imperfect class-conditional probability estimates affect the expected utility of downstream decisions. We provided an analytical expression for the regret in expected utility caused by miscalibration, along with upper and lower bounds on the regret due to grouping loss. Our experiments show that these quantities better capture the potential gains from recalibration and post-training in expected utility compared to common metrics. In the future, it would be of interest to extend theses results to the multiclass setting.

\subsubsection*{Acknowledgments}
We thank Thomas Moreau for valuable discussions that helped refine the derivations of the bounds.

\bibliography{main}

\section*{Checklist}

 \begin{enumerate}

 \item For all models and algorithms presented, check if you include:
 \begin{enumerate}
   \item A clear description of the mathematical setting, assumptions, algorithm, and/or model. Yes. For the theoretical part, all assumptions, definitions are specified and proofs given in appendix. For the experimental part, a detailed procedure is given in \autoref{sec:details} and a python code repository enables reproducing the results (\url{https://github.com/aperezlebel/decision_suboptimal_classifiers}).
   \item An analysis of the properties and complexity (time, space, sample size) of any algorithm. Not Applicable.
   \item (Optional) Anonymized source code, with specification of all dependencies, including external libraries. Yes.
 \end{enumerate}

 \item For any theoretical claim, check if you include:
 \begin{enumerate}
   \item Statements of the full set of assumptions of all theoretical results. Yes. All theorems have their assumptions defined.
   \item Complete proofs of all theoretical results. Yes. All proofs are given in appendix.
   \item Clear explanations of any assumptions. Yes.
 \end{enumerate}

 \item For all figures and tables that present empirical results, check if you include:
 \begin{enumerate}
   \item The code, data, and instructions needed to reproduce the main experimental results (either in the supplemental material or as a URL). Yes. The python code repository is given in the supplemental material.
   \item All the training details (e.g., data splits, hyperparameters, how they were chosen). Yes. The training details are given in \autoref{sec:details}.
         \item A clear definition of the specific measure or statistics and error bars (e.g., with respect to the random seed after running experiments multiple times). Not Applicable
         \item A description of the computing infrastructure used. (e.g., type of GPUs, internal cluster, or cloud provider). Yes. Details given in \autoref{sec:details}.
 \end{enumerate}

 \item If you are using existing assets (e.g., code, data, models) or curating/releasing new assets, check if you include:
 \begin{enumerate}
   \item Citations of the creator If your work uses existing assets. Yes. All datasets and pre-trained models used are refered in \autoref{tab:ds_translation} and \autoref{tab:model_translation}.
   \item The license information of the assets, if applicable. Not Applicable.
   \item New assets either in the supplemental material or as a URL, if applicable. Not Applicable.
   \item Information about consent from data providers/curators. Not Applicable
   \item Discussion of sensible content if applicable, e.g., personally identifiable information or offensive content. Not Applicable.
 \end{enumerate}

 \item If you used crowdsourcing or conducted research with human subjects, check if you include:
 \begin{enumerate}
   \item The full text of instructions given to participants and screenshots. Not Applicable.
   \item Descriptions of potential participant risks, with links to Institutional Review Board (IRB) approvals if applicable. Not Applicable.
   \item The estimated hourly wage paid to participants and the total amount spent on participant compensation. Not Applicable.
 \end{enumerate}

 \end{enumerate}

\clearpage
\appendix
\onecolumn
\part*{Appendix}

\section{Proposed evaluation procedure}
\paragraph{A procedure to evaluate whether to post-train or not for given utilities}
Our findings suggest a new model evaluation procedure for batch decision-making under uncertainty. For a given trained model, it is important to determine whether post-training will benefit a decision task, what type of post-training is sufficient (recalibration or more advanced methods), and what the potential gain might be.
In practice, with large or regulatory-constrained models (\eg{} foundation models \citep{Zhou2023}, health models \citep{Collins2012} or both \citep{Moor2023}),
it is prohibitive to apply a post-train-and-find-out strategy.
Guiding post-training upfront is thus valuable. Concretely (\autoref{fig:procedure}), we propose to first measure the model's regret $\rht$ for the specified decision task using $\rhthat = \rclhthat + \rglhhat$, and to assess whether the model is utility-suboptimal ($\rht > 0$). The regret decomposition then informs on whether the regret comes mainly from miscalibration ($\rht \approx \rclht$). This enables selecting the appropriate post-training method and avoiding unnecessary, costly methods. The strong correlation between the estimated regret $\rhthat$ and the effective gain of post-training enables a cost-benefit analysis: is the potential utility gain worth the cost of post-training the model? Note that this cost not only includes computational costs, but also the cost of changing the production model and the potential risks of deploying a new model. If the analysis favors post-training, we start-over the procedure to check whether the post-training was effective (post-training can sometimes undermine the original model). If the analysis rejects post-training, or if the model was not suboptimal, the decision-maker assesses whether the achieved utility is satisfying for their task.
If not, gathering more informative features may then reduce aleatoric uncertainty and potentially improve the utility.
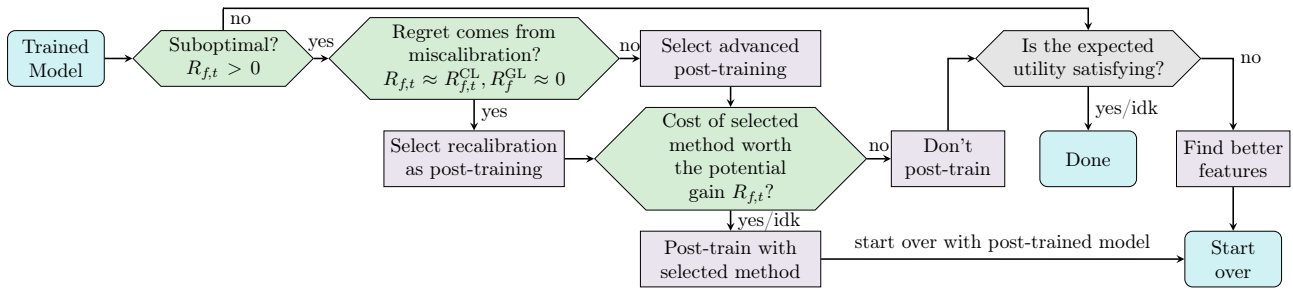
\begin{figure}[ht]
    \centering
    \resizebox{\linewidth}{!}{%
    \begin{tikzpicture}[node distance=2cm]
        \input{img/fig_flowchart.tex}
    \end{tikzpicture}
    }
    \caption{\textbf{Proposed model evaluation procedure.} From a trained model, our work enables three major steps: assessing whether a model is suboptimal for the decision task at hand, quantifying the expected regret, and finding its origin: miscalibration or not (green boxes). This enables disambiguating between when to improve the model (post-train) and when to improve the data (finding better features).\looseness=-1
    }
    \label{fig:procedure}
\end{figure}

\clearpage

\section{Theoretical results}
\subsection{Parametrization of decision rules}
\begin{lemma}[Parametrization of decision rules]
\label{lem:decision:parametrized}
The set of binary decision rules $\Xc^{\{0, 1\}}$ can be parametrized with a function $p : \Xc \to [0, 1]$ and a threshold $t \in [0, 1]$:
\begin{equation}
    \Xc^{\{0, 1\}} = \cbr{\delta_{p, t} : x \mapsto \1{p(x) \geq t} \text{ where } p : \Xc \to [0, 1], t \in [0, 1]}
\end{equation}
\end{lemma}
\begin{proof}[Proof of \autoref{lem:decision:parametrized}]
First let's show the inclusion of left in right. Let $\delta \in \Xc^{\{0, 1\}}$. Let's show that there exists a function $p : \Xc \to [0, 1]$ and a threshold $t \in [0, 1]$ such that $\delta = \delta_{p, t}$.

Define $p : x \mapsto \delta(x)$ and $t = \frac{1}{2}$. Then for all $x \in \Xc$, $\delta_{p, t}(x) = \1{p(x) \geq t} = \1{\delta(x) \geq \tfrac{1}{2}} = \delta(x)$.

The inclusion of right in left is trivial because all $\delta_{p, t}$ are decision rules from $\Xc$ to $\{0, 1\}$.
\end{proof}

\subsection{$\mathcal{D}_{\hh}$ contains all decisions relying on $\hh$ only}
\begin{lemma}[$\mathcal{D}_{\hh}$]
\label{lem:dh}
    $\mathcal{D}_{\hh} \triangleq \cbr{\delta_{g(\hh), t} : \quad g: [0, 1] \to [0, 1],\; t \in [0, 1]}$ contains all possible decision rules taking $\hh$ as input:
    \begin{equation}
        \mathcal{D}_{\hh} = \cbr{x \mapsto d(\hh(x)) : d \in [0, 1]^{\{0, 1\}}}
    \end{equation}
\end{lemma}
\begin{proof}[Proof of \autoref{lem:dh}]
Let $\delta \in \mathcal{D}_{\hh}$. By definition of $\mathcal{D}_{\hh}$, there exists $g : [0, 1] \to [0, 1]$ and $t \in [0, 1]$ such that $\delta(x) = \1{g(\hh(x)) \geq t}$ for all $x \in \Xc$. Define $d : p \mapsto \1{g(p) \geq t}$. Then $d \in [0, 1]^{\{0, 1\}}$ and $\delta = x \mapsto d(\hh(x))$.

Let $\delta \in \cbr{x \mapsto d(\hh(x)) : d \in [0, 1]^{\{0, 1\}}}$. Thus, there exists $d \in [0, 1]^{\{0, 1\}}$ such that $\delta = x \mapsto d(\hh(x))$. Define $g : p \mapsto d(p)$ and $t = \frac{1}{2}$. Then $\delta(x) = g(\hh(x)) = \1{g(\hh(x)) \geq t}$. Hence $\delta \in \mathcal{D}_{\hh}$.
\end{proof}

\subsection{Expression of expected utility}
\begin{lemma}[Expected utility]
    \label{lem:expected-utility}
    Let $\delta \in \Xc \to {\{0, 1\}}$ and $x \in \Xc$. Then the expected utility of $\delta$ conditional to $x$ is:
    \begin{equation}
        \EU(\delta, x) = \unorm\delta(x)(\hstar(x)-t\sstar)\!+\!\hstar(x)U_{01}\!+\!(1\!-\!\hstar(x))U_{00}
    \end{equation}
\end{lemma}
\begin{corollary}[Difference in expected utilities]
\label{cor:expected-utility:diff}
Let $\delta_1, \delta_2 \in \Xc \to {\{0, 1\}}$ and $x \in \Xc$. Then the difference in expected utilities of $\delta_1$ and $\delta_2$ conditional to $x$ is:
\begin{equation}
    \EU(\delta_2, x) - \EU(\delta_1, x) = \unorm(\delta_2(x) - \delta_1(x))(\hstar(x)-t\sstar)
\end{equation}
\end{corollary}
\begin{proof}[Proof of \autoref{lem:expected-utility} and \autoref{cor:expected-utility:diff}]
Let $\delta \in \Xc \to {\{0, 1\}}$ and $x \in \Xc$.

First note that:
$U[\delta(x), 1] = \begin{cases}
    U_{11} & \text{if } \delta(x) = 1\\
    U_{01} & \text{otherwise}
\end{cases} = \delta(x)(U_{11} - U_{01}) + U_{01}$.

Similarly, $U[\delta(x), 0] = \begin{cases}
    U_{10} & \text{if } \delta(x) = 1\\
    U_{00} & \text{otherwise}
\end{cases} = \delta(x)(U_{10} - U_{00}) + U_{00}$.

Then the expected utility of $\delta$ conditional to $x$ writes:
\begin{align*}
    \EU(\delta, x)
    & = \espk{U_{\delta(X), Y}}{X=x}& \text{\autoref{eq:eu:pointwise}}\\
    & = \hstar(x)U_{\delta(x), 1} + (1 - \hstar(x))U_{\delta(x), 0}&\\
    & = \hstar(x)(\delta(x)(U_{11} - U_{01}) + U_{01})&\\
    & \qquad + (1 - \hstar(x))(\delta(x)(U_{10} - U_{00}) + U_{00})&\\
    & = \delta(x)(\hstar(x)(U_{00}\!-\!U_{10}\!+\!U_{11}\!-\!U_{01})\!-\!(U_{00}\!-\!U_{01}))&\\
    & \qquad +\hstar(x)U_{01}\!+\!(1\!-\!\hstar(x))U_{00}&\\
    & = \unorm\delta(x)(\hstar(x)-t\sstar)\!+\!\hstar(x)U_{01}\!+\!(1\!-\!\hstar(x))U_{00}&\text{Def of $\unorm$ and $t\sstar$ (\autoref{eq:optimal_classifier})}\\
\end{align*}
Hence for all $\delta_1, \delta_2 \in \Xc^{\{0, 1\}}$ and $x \in \Xc$:
\begin{align}
    \EU(\delta_2, x) - \EU(\delta_1, x)
    & = \unorm(\delta_2(x) - \delta_1(x))(\hstar(x)-t\sstar)&
\end{align}
\end{proof}

\subsection{Best decision over $\mathcal{D}_{\hh}$}
\label{sec:proof:best-decision:dh}
\thmBestDh*
\begin{proof}[Proof of \autoref{prop:decision:est}]

Let $\delta \in \mathcal{D}_\hh$. Let's show that $\EU(\delta_{\chch, t\sstar}|p) \geq \EU(\delta|p)$.

By definition of $\mathcal{D}_\hh$, there exist $g : [0, 1] \to [0, 1]$ and $t \in [0, 1]$ such that $\delta = x \mapsto \1{g(\hh(x)) \geq t}$.

Let $p \in \supp{\hh(X)}$.
\begin{align}
    & \EU(\delta_{\chch, t\sstar}|p)\!-\!\EU(\delta|p) &\\
    & = \espk{\EU(\delta_{\chch, t\sstar}, X) - \EU(\delta, X)}{\hh(X)=p}& \text{Definition of~}\EU(\cdot|p)\\
    & = \unorm\espk{(\delta_{\chch, t\sstar}(X) - \delta(X))(\hstar(X) - t\sstar)}{\hh(X)=p}
    & \text{\autoref{cor:expected-utility:diff}}\\
    & = \unorm\espk{(\1{\ch(\hh(X)) \geq t\sstar} - \1{g(\hh(X)) \geq t})(\hstar(X) - t\sstar)}{\hh(X)=p} & \text{Def. of $\delta_{\chch, t\sstar}$ and $\delta$}\\
    & = \unorm(\1{\ch(\hh(X)) \geq t\sstar} - \1{g(\hh(X)) \geq t})(\espk{\hstar(X)}{\hh(X)} - t\sstar) & \text{\quad$\delta_{\chch, t\sstar}$ and $\delta$ func. of $\hh(X)$}\\
    & = \unorm(\1{\ch(\hh(X)) \geq t\sstar} - \1{g(\hh(X)) \geq t})(\ch(\hh(X)) - t\sstar) & \text{Def. of $\ch$ (\ref{eq:calibration-curve})}\label{eq:proof:rcl}\\
    & = \unorm\1{\ch(\hh(X)) \geq t\sstar}\1{g(\hh(X)) < t}(\ch(\hh(X)) - t\sstar) & \text{}\\
    & \qquad + \unorm\1{\ch(\hh(X)) < t\sstar}\1{g(\hh(X)) \geq t}(t\sstar - \ch(\hh(X))) & \text{}\\
    & = \unorm(\1{\ch(\hh(X)) \geq t\sstar}\1{g(\hh(X)) < t} + \1{\ch(\hh(X)) < t\sstar}\1{g(\hh(X)) \geq t}) &\\
    & \qquad \times|\ch(\hh(X)) - t\sstar| & \text{}\\
    & = \unorm|\1{\ch(\hh(X)) \geq t\sstar} - \1{g(\hh(X)) \geq t}||\ch(\hh(X)) - t\sstar| &\label{eq:proof:rcl2}\\
    & \geq 0 & \text{}
\end{align}
Hence:
\begin{equation}
    \delta_{\cch, t^\star} \in \argmax{\delta \in \mathcal{D}_{\hh}} \EU(\delta|p)
\end{equation}
\end{proof}

\subsection{Calibration regret $\rclht$}
\subsubsection{Expression}
\label{sec:proof:rcl:expression}
\propRCLExpression*
\begin{proof}[Proof of \autoref{prop:rcl:expression}]
    Let $\delta \in \mathcal{D}_{\hh}$.
    By definition of $\mathcal{D}_{\hh}$, there exist $g : [0, 1] \to [0, 1]$ and $t \in [0, 1]$ such that $\delta = x \mapsto \1{g(\hh(x)) \geq t}$.
    Let $p \in \supp{\hh(X)}$.

    \begin{align}
        \rclht(p)
        & = \EU(\delta_{\cch, t\sstar}|p) - \EU(\delta_{\hh, t}|p) &\text{Definition of~}\rclht(p)\\
        & = \espk{\EU(\delta_{\ch(\hh), t\sstar}, X)\!-\!\EU(\delta_{\hh,t}, X)}{\hh(X) = p} &\text{Definition of~}\EU(\cdot|p)\\
        & = \unorm(\1{\ch(p) \geq t\sstar} - \1{g(p) \geq t})(\ch(p) - t\sstar)
        & \text{\autoref{eq:proof:rcl}} \label{eq:proof:rcl3}\\
        & = \unorm|\1{\ch(p) \geq t\sstar} - \1{g(p) \geq t}||\ch(p) - t\sstar| & \text{\autoref{eq:proof:rcl2}}
    \end{align}
\end{proof}

\subsubsection{Zero calibration regret for all utilities implies calibration}
\begin{proposition}
    \label{prop:rcl:0:forall:t}
    For all $p \in \supp{\hh(X)}$,
    \begin{align}
        \rclhtstar(p) = 0, \quad \forall t\sstar \in [0, 1] \; \Longleftrightarrow \; \ch(p) = p.
    \end{align}
\end{proposition}
\begin{proof}[Proof of \autoref{prop:rcl:0:forall:t}]
Let $p \in \supp{\hh(X)}$. Suppose for all $t\sstar \in [0, 1], \rclhtstar(p) = 0$. Let's show that $\ch(p) = p$. By contradiction, suppose $\ch(p) \neq p$. Let's show that there exists $t\sstar \in [0, 1]$ such that $\rclhtstar(p) > 0$. Take $t\sstar = \tfrac{1}{2}(\ch(p) + p)$. Then $t\sstar \in [0, 1]$ and $|\1{\ch(p) \geq t\sstar} - \1{p \geq t\sstar}| = 1$ and $|\ch(p) - t\sstar| = \tfrac{1}{2}|\ch(p) - p| > 0$. Hence $\rclhtstar(p) > 0$ which proves the contradiction.
Now suppose $\ch(p) = p$. Then $|\1{\ch(p) \geq t\sstar} - \1{p \geq t\sstar}| = |\1{p \geq t\sstar} - \1{p \geq t\sstar}| = 0$ and $\rclhtstar(p) = \unorm|\1{\ch(p) \geq t\sstar} - \1{p \geq t\sstar}||\ch(p) - t\sstar| = 0$ (\autoref{prop:rcl:expression}).
\end{proof}

\subsubsection{Adjusting the threshold to address the calibration regret}
\label{sec:proof:th}
\thmTH*
\begin{proof}[Proof of \autoref{prop:th}]
Let $t\sstar \in [0, 1]$. Suppose $\ch$ is monotonic non-decreasing. Let $p \in \supp{\hh(X)}$. Then $p \geq t_\hh \Leftrightarrow \ch(p) \geq \ch(t_\hh)$ because $\ch$ is non-decreasing. When $t_\hh \in \ch^{-1}(\{t\sstar\})$, then $\ch(t_\hh) = t\sstar$ and $p \geq t_\hh \Leftrightarrow \ch(p) \geq t\sstar$ hence $\1{\ch(p) \geq t\sstar} = \1{p \geq t_\hh}$.
When $t_\hh = \inf \{t: \ch(t) \geq t\sstar \}$, then $\ch(t_\hh) \geq t\sstar$ by definition. Let's show that we still have $p \geq t_\hh \Leftrightarrow \ch(p) \geq t\sstar$. Suppose $p \geq t_\hh$. Then $\ch(p) \geq \ch(t_\hh) \geq t\sstar$. Suppose $\ch(p) \geq t\sstar$. Then $p \in \{t: \ch(t) \geq t\sstar \}$. But $t_\hh$ is the smallest element of this set. So $p \geq t_\hh$. Hence $\1{\ch(p) \geq t\sstar} = \1{p \geq t_\hh}$.
Hence:
\begin{align}
    \rclhtstar(p)
    & = \unorm(\1{\ch(p) \geq t\sstar} - \1{p \geq t_\hh})(\ch(p) - t\sstar) = 0 & \text{\autoref{eq:proof:rcl3}}
\end{align}
\end{proof}

\subsection{Lemma on the variance}

\begin{lemma}[Bhatia-Davis inequality \citep{Bhatia2000}]
\label{lem:bhatia-davis}
Let $m, M \in \R$ such that $m \leq M$. Let $Z$ a random variable valued in $[m, M]$. Then the variance of $Z$ is upper bounded by:
\begin{equation*}
    \var{Z} \leq (M - \esp{Z})(\esp{Z} - m)
\end{equation*}
\end{lemma}

\begin{lemma}[Bounds on the variance of a random variable in {[}0, 1{]}]
    \label{lem:var}
    Let $t \in [0, 1]$.
    Let $Z$ be a random variable valued in $[0, 1]$.
    Define $\cp \triangleq \espk{Z}{Z \geq t}$, $\cm \triangleq \espk{Z}{Z < t}$ and $w \triangleq \PP{Z \geq t}$.

    When both $\cp$ and $\cm$ are defined, the expectation of $Z$ satisfies:
    \begin{align}
        \esp{Z} = w\cp + (1-w)\cm\label{eq:exp:decomposition}
    \end{align}
    and the variance of $Z$ is bounded by:
    \begin{align}
        \var{Z} & \geq w(1-w)(\cp - \cm)^2 & \text{Variance lower bound}\label{eq:var:lb}\\
        \var{Z} & \leq w(\cp - t) + \esp{Z}(t-\esp{Z}) & \text{Variance upper bound}\label{eq:var:ub}
    \end{align}

    When only one of $\cp$ and $\cm$ is defined,
    the variance of $Z$ is bounded by:
    \begin{align}
        \var{Z} & \geq 0 & \text{Variance lower bound}\label{eq:var:lb2}\\
        \var{Z} & \leq \begin{cases}
            \esp{Z}(t-\esp{Z}) & \text{if } \esp{Z} < t\\
            (1 - \esp{Z})(\esp{Z}-t) & \text{if } \esp{Z} \geq t
        \end{cases}& \text{Variance upper bound}\label{eq:var:ub2}
    \end{align}

    \end{lemma}
    \begin{proof}[Proof of \autoref{lem:var}]
    Let $t \in [0, 1]$.
    Let $Z$ be a random variable valued in $[0, 1]$.
    Define $\cp \triangleq \espk{Z}{Z \geq t}$, $\cm \triangleq \espk{Z}{Z < t}$ and $w \triangleq \PP{Z \geq t}$.

    The case where only one of $\cp$ and $\cm$ is defined amounts to the Bhatia-Davis inequality (\autoref{lem:bhatia-davis}) on $Z$ which is valued in $[0, t]$ if $\cp$ is undefined and in $[t, 1]$ if $\cm$ is undefined. This proves \autoref{eq:var:ub2}.

    From now we suppose that both $\cp$ and $\cm$ are defined.

    \paragraph{Expectation.}The expectation of $Z$ satisfies:
    \begin{align*}
        \esp{Z}
        & = \esp{\espk{Z}{\1{Z \geq t}}} & \text{Law of total expectation}\\
        & = \PP{Z \geq t}\espk{Z}{Z \geq t} + \PP{Z < t}\espk{Z}{Z < t} & \text{}\\
        & = w\cp + (1 - w)\cm& \text{Def. of $\cp, \cm$ and $w$}
    \end{align*}
    This proves \autoref{eq:exp:decomposition}.
    Next we show the bounds on the variance.

    \paragraph{Variance.} For convinience, we note: $\mu \triangleq \esp{Z}$, $v^+ \triangleq \var{Z \geq t}$ and $v^- \triangleq \vark{Z}{Z < t}$. The variance of $Z$ satisfies:
    \begin{align}
        \var{Z}
        & = \esp{\vark{Z}{\1{Z \geq t}}} + \var{\espk{Z}{\1{Z \geq t}}}
        & \text{Law of total variance}\nonumber\\
        & = wv^+ + (1 - w)v^- + w(\cp - \mu)^2 + (1 - w)(\cm - \mu)^2\label{eq:var:decomposition}
    \end{align}
    From this expression, we derive both the upper and lower bounds, using either the positivity of $v^+$ and $v^-$ or the Bhatia-Davis inequality (\autoref{lem:bhatia-davis}) on $v^+$ and $v^-$.

    \paragraph{Upper bound.}First, we show the upper bound. Using the fact that $Z | Z \geq t$ is in $[t, 1]$ and $Z | Z  < t$ is in $[0, t]$, the Bhatia-Davis inequality gives:
    \begin{align}
        v^+ & \leq (1 - \cp)(\cp - t)\\
        v^- & \leq (t - \cm)\cm
    \end{align}
    Hence, \autoref{eq:var:decomposition} gives:
    \begin{align*}
        \var{Z}
        & \leq w(1 - \cp)(\cp - t) + (1 - w)(t - \cm)\cm&\\
        & \quad + w(\cp - \mu)^2 + (1 - w)(\cm - \mu)^2&\\
        & = w(\cp - t + t\cp - 2\cp\mu) + (1 - w)\cm(t - 2\mu) + \mu^2&\\
        & = w(\cp - t + t\cp - 2\cp\mu) + (\mu - w\cp)(t - 2\mu) + \mu^2& \text{Using \autoref{eq:exp:decomposition}}\\
        & = w(\cp - t) + \mu(t - 2\mu) + \mu^2& \\
        & = w(\cp - t) + \mu(t - \mu)&
    \end{align*}
    This proves the upper bound (\autoref{eq:var:ub})

    \paragraph{Lower bound.} Now, we prove the lower bound. First notice that \autoref{eq:exp:decomposition} gives:
    \begin{align}
        \cm - \mu & = w(\cm - \cp) \label{eq:mu:diff1}\\
        \cp - \mu & = (1-w)(\cp - \cm) \label{eq:mu:diff2}
    \end{align}

    Using the positivity of $v^+$ and $v^-$ in \autoref{eq:var:decomposition}, we have:
    \begin{align}
        \var{Z}
        & \geq w(\cp - \mu)^2 + (1 - w)(\cm - \mu)^2&\nonumber\\
        & = w(1-w)^2(\cp - \cm)^2 + (1 - w)w^2(\cp - \cm)^2&\nonumber\text{Using \autoref{eq:mu:diff1} and \ref{eq:mu:diff2}}\\
        & = w(1-w)(\cp - \cm)^2\label{eq:var:lb:proof}&
    \end{align}
    This proves the lower bound (\autoref{eq:var:lb}).

    Note that the results and proof holds for $Z | V$ for any random variable $V$, by conditioning all expectations and variance by $V$.
    \end{proof}

\subsection{$\rgl$ regret reformulation}
\begin{restatable}[Grouping loss induced regret]{lemma}{thmRGLExpression}
    \label{thm:rgl:expression}
    For all $p \in [0, 1]$ where $\rglh$ is defined, then:
    \begin{align}
        \rglh(p) & = \unorm\,\espk{(\1{\hstar(X) \geq t\sstar} - \1{c(p) \geq t\sstar})(\hstar(X) - t\sstar)}{h(X)\!=\!p}.
    \end{align}
\end{restatable}
\begin{proof}[Proof of \autoref{thm:rgl:expression}]
Let $p \in [0, 1]$ where $\rglh$ is defined.

\begin{align}
    \rglh(p)
    & = \espk{\EU(\deltastar, X) - \EU(\delta_{c_\hh(\hh), t\sstar}, X)}{\hh(X)\!=\!p} & \text{}\\
    & = \unorm\espk{(\deltastar(X) - \delta_{c_\hh(\hh), t\sstar}(X))(\hstar(X)-t\sstar)}{\hh(X) = p} & \text{\autoref{cor:expected-utility:diff}}\\
    & = \unorm\espk{(\1{\hstar(X) \geq t\sstar} - \1{c_\hh(p) \geq t\sstar}(X))(\hstar(X)-t\sstar)}{\hh(X) = p} & \text{\autoref{eq:decision:rule}}
\end{align}
\end{proof}

\begin{lemma}[Regret reformulation]
\label{lem:regret:reform}
    Let $p \in \supp{\hh(X)}$.
    Define $w(p) \triangleq \PPG{\hstar(X) \geq t\sstar}{\hh(X) = p}$ and $c^+(p) \triangleq \espk{\hstar(X)}{\hstar(X) \geq t\sstar, \hh(X) = p}$. Then, the grouping loss induced regret $\rglh$ can be written as:
    \begin{equation}
        \rglh(p) = \begin{cases}
            \unorm\pr{w(p)(c^+(p) - t\sstar)-\1{\ch(p) \geq t\sstar}(\ch(p)-t\sstar)}
            & \text{if } 0 < w(p) < 1.\\
            0 & \text{otherwise}.
        \end{cases}
    \end{equation}
\end{lemma}
\begin{proof}[Proof of \autoref{lem:regret:reform}]
    Let $p \in \supp{\hh(X)}$.

    Suppose $0 < \PPG{\hstar(X) \geq t\sstar}{\hh(X) = p} < 1$. Then $c^+(p)$ is defined. We have:
\begin{align}
    & \rglh(p) &\\
    & = \unorm\,\espk{(\1{\hstar(X) \geq t\sstar} - \1{c(p) \geq t\sstar})(\hstar(X) - t\sstar)}{h(X)\!=\!p}
    & \text{\autoref{lem:regret:reform}}\\
    & = \unorm\,\espk{\1{\hstar(X) \geq t\sstar}(\hstar(X) - t\sstar)}{h(X)\!=\!p} & \label{eq:proof:rgl1}\\
    & \qquad - \unorm\,\1{c(p) \geq t\sstar}(\espk{\hstar(X)}{\hh(X) = p} - t\sstar)
    & \text{}\\
    & = \unorm\,\espk{\espk{\1{\hstar(X) \geq t\sstar}(\hstar(X) - t\sstar)}{\hstar\!(X)\!\geq\!t\sstar\!, \hh(X)}}{h(X)\!=\!p} & \text{Total exp.}\\
    & \qquad - \unorm\,\1{c(p) \geq t\sstar}(\ch(p) - t\sstar)
    & \text{Def. of $\ch$ (\autoref{eq:calibration-curve})}\\
    & = \unorm\,\espk{\1{\hstar(X) \geq t\sstar}(\espk{\hstar(X)}{\hstar\!(X)\!\geq\!t\sstar\!, \hh(X)} - t\sstar)}{h(X)\!=\!p} & \text{}\\
    & \qquad - \unorm\,\1{c(p) \geq t\sstar}(\ch(p) - t\sstar)
    & \text{}\\
    & = \unorm\,\espk{\1{\hstar(X) \geq t\sstar}(c^+(p) - t\sstar)}{h(X)\!=\!p} & \text{Def. of $c^+$}\\
    & \qquad - \unorm\,\1{c(p) \geq t\sstar}(\ch(p) - t\sstar)
    & \text{}\\
    & = \unorm\,\espk{\1{\hstar(X) \geq t\sstar}}{h(X)\!=\!p}(c^+(p) - t\sstar) & \text{}\\
    & \qquad - \unorm\,\1{c(p) \geq t\sstar}(\ch(p) - t\sstar)
    & \text{}\\
    & = \unorm\,w(p)(c^+(p) - t\sstar) & \text{Def. of $w$}\\
    & \qquad - \unorm\,\1{c(p) \geq t\sstar}(\ch(p) - t\sstar)
    & \text{}
\end{align}

Suppose $\PPG{\hstar(X) \geq t\sstar}{\hh(X) = p} = 1$. Then the above proof is still valid and $w(p) = 1$, $c^+(p) = \ch(p)$ and $\ch(p) \geq t\sstar$. Hence $\rglh(p) = 0$.

Suppose $\PPG{\hstar(X) \geq t\sstar}{\hh(X) = p} = 0$. Then $\espk{\1{\hstar(X) \geq t\sstar}(\hstar(X) - t\sstar)}{h(X)\!=\!p} = 0$ and $\1{\ch(p) \geq t\sstar} = 0$. Using \autoref{eq:proof:rgl1} we have $\rglh(p) = 0$.

\end{proof}

\subsection{Grouping regret lower bound}
\label{proof:regret:lb}

\thmRegretLB*
\begin{proof}[Proof of \autoref{thm:regret:lb}]
    Let $p \in \supp{\hh(X)}$.\\
    Suppose $0 < \PPG{\hstar(X) \geq t\sstar}{\hh(X) = p} < 1$.

    Define $w(p) \triangleq \PPG{\hstar(X) \geq t\sstar}{\hh(X)=p}$ and $c^+(p) \triangleq \espk{\hstar(X)}{\hstar(X) \geq t\sstar, \hh(X)=p}$. \autoref{lem:regret:reform} gives:
    \begin{equation}
        \rglh(p) = \unorm\pr{w(p)(c^+(p) - t\sstar) - \1{\ch(p) \geq t\sstar}(\ch(p) - t\sstar)}
    \end{equation}

    We apply the upper bound of \autoref{lem:var} to $Z \triangleq \hstar(X) | \hh(X)$.
    Hence:
    \begin{align}
        \vark{\hstar(X)}{\hh(X)=p}
        & \leq w(p)(c^+(p) - t\sstar) + \ch(p)(t\sstar-\ch(p))\\
        & = \unorm\rglh(p) + \1{\ch(p) \geq t\sstar}(\ch(p) - t\sstar) + \ch(p)(t\sstar-\ch(p))\\
        & = \unorm\rglh(p) + \1{\ch(p) \geq t\sstar}(\ch(p) - t\sstar)(1 - \ch(p)) \\
        & \qquad + \1{\ch(p) < t\sstar}\ch(p)(t\sstar - \ch(p))\\
        & = \unorm\rglh(p) + \vmin(p)
        \intertext{Hence:}
        \rglh(p)
        & \geq \unorm(\vfhp - \vmin(p))
    \end{align}

    Since $\rglh(p) \geq 0$, we have:
    \begin{align}
        \rglh(p) \geq \unorm\left[\vfhp - \vmin(p)\right]_+
    \end{align}

    Suppose $\PPG{\hstar(X) \geq t\sstar}{\hh(X) = p} \in \{0, 1\}$. Then $\rglh(p) = 0$ (\autoref{lem:regret:reform}). Let's show that $\lglh(p) = 0$.
    According to \autoref{lem:var}, we have $\vfhp \leq \vmin(p)$ (\autoref{eq:var:ub2}). Then $\lglh(p) = 0$. Then $\rglh(p) = \lglh(p)$.

    \paragraph{Tightness of the regret lower bound $\lglh$.} Following the proof of the regret lower bound, we see that the bound is obtained using the upper bound of the lemma on the variance (\autoref{lem:var}). The proof of \autoref{lem:var} gives an idea on how to achieve the equality. Taking a distribution where $v^+$ and $v^-$ saturates the Bhatia-Davis inequality is a good candidate. That is, a distribution of diracs in $\{0, t, 1\}$ with appropriate weights.

    Below we explicit distributions achieving the equality between the regret $\rglh$ and the lower bound $\lglh$. For convenience, we work with a random variable $Z$ valued in $[0,1]$ which we will later link to the true probabilities $\PPG{\hstar(X)}{\hh(X) = p}$.

    Let $t \in (0,1)$. Let $c \in [0, 1]$ and $v \in [0, c(1 - c)]$. The Bhatia-Davis inequality shows that these represent the sets of admissible values of mean and variances of a random variable valued in $[0,1]$. Note $R_Z \triangleq \unorm\esp{(\1{Z\geq t} - \1{\esp{Z}\geq t})(Z - t)}$ and $L_Z \triangleq \unorm[\var{Z} - \svmin]_+$ with:
    \begin{equation}
        \svmin \triangleq \begin{cases}
            c(t - c) & \text{if } c < t\\
            (1 - c)(c - t) & \text{if } c \geq t
        \end{cases}.
    \end{equation}
    $R_Z$ represents the regret $\rglh$ (\autoref{thm:rgl:expression})
    and $L_Z$ its lower bound $\lglh$ (\autoref{thm:regret:lb}) when $c = \esp{Z}$.

    We now show that there exist distributions $\mathbb{P}_Z$ of $Z$ such that $\esp{Z} = c$, $\var{Z} = v$ and $R_Z = L_Z$. Depending of the value of $c$ and $v$, we explicit 4 distributions of $Z$ that achieve the equality.

    \paragraph{Case 1: $c < t$ and $v < \svmin$.} Define $Z_1$ distributed as:
    $\mathbb{P}_{Z_1} = w_0\delta_0 + w_c\delta_c + w_t\delta_t$
    with $w_c = 1 - \frac{v}{\svmin}$, $w_t = \frac{c}{t}\frac{v}{\svmin}$ and $w_0 = 1 - w_c - w_t$.
    \begin{align}
        \esp{Z_1}
        & = w_cc+w_tt\\
        & = c\\
        \var{Z_1}
        & = w_0c^2 + w_t(t - c)^2\\
        & = (1 - w_1)c^2 - w_t t(t-2c)\\
        & = \tfrac{v}{\svmin}c^2 - \tfrac{v}{\svmin}c(t-2c)\\
        & = \tfrac{v}{\svmin}c(t-c)\\
        & = v\\
        R_{Z_1}
        & = \unorm\esp{\1{Z_1\geq t}(Z_1 - t)}\\
        & = 0\\
        & = L_{Z_1}
    \end{align}

    \paragraph{Case 2: $c \geq t$ and $v < \svmin$.} Define $Z_2$ distributed as:
    $\mathbb{P}_{Z_2} = w_t\delta_t + w_c\delta_c + w_1\delta_1$
    with $w_c = 1 - \frac{v}{\svmin}$, $w_1 = \frac{c-t}{1-t}\frac{v}{\svmin}$ and $w_t = 1 - w_c - w_t$.

    \begin{align}
        \esp{Z_2}
        & = w_tt + w_cc + w_1\\
        & = t + w_c(c-t) + w_1(1-t)\\
        & = t + (1-\tfrac{v}{\svmin})(c-t) + \tfrac{c-t}{1-t}\tfrac{v}{\svmin}(1-t)\\
        & = c\\
        \var{Z_2}
        & = w_t(t - c)^2 + w_1(1-c)^2\\
        & = (1 - w_1)c^2 - w_t t(t-2c)\\
        & = (t-c)^2 - w_c(t-c)^2 + w_1(1-t)(1+t-2c)\\
        & = \tfrac{v}{\svmin}(c-t)(1-c)\\
        & = v\\
        R_{Z_2}
        & = \unorm\esp{(\1{Z_2\geq t} - 1)(Z_2 - t)}\\
        & = \unorm\esp{\1{Z_2 < t}(t - Z_2)}\\
        & = 0\\
        & = L_{Z_2}
    \end{align}

    \paragraph{Case 3: $c < t$ and $v \geq \svmin$.} Define $Z_3$ distributed as:
    $\mathbb{P}_{Z_3} = w_0\delta_0 + w_t\delta_t + w_1\delta_1$
    with $w_1 = \frac{1}{1-t}(v - \svmin)$, $w_t = \frac{c-w_1}{t}$  and $w_0 = 1 - w_t - w_1$.
    \begin{align}
        \esp{Z_3}
        & = w_tt + w_1\\
        & = c\\
        \var{Z_3}
        & = w_0c^2 + w_t(t - c)^2 + w_1(1-c)^2\\
        & = c^2 + w_tt(t-2c) + w_1(1-2c)\\
        & = c^2 + (c - w_1)(t-2c) + w_1(1-2c)\\
        & = c(t-c) + w_1(1-t)\\
        & = c(t-c) + v - \svmin\\
        & = v\\
        R_{Z_3}
        & = \unorm\esp{\1{Z_3\geq t}(Z_3 - t)}\\
        & = w_1(1-t)\\
        & = v - \svmin\\
        & = L_{Z_3}
    \end{align}

    \paragraph{Case 4: $c \geq t$ and $v \geq \svmin$.} Define $Z_4$ distributed as:
    $\mathbb{P}_{Z_4} = w_0\delta_0 + w_t\delta_t + w_1\delta_1$
    with $w_0 = \frac{1}{t}(v - \svmin)$, $w_t = \frac{1-c-w_0}{1-t}$  and $w_1 = 1 - w_t - w_0$.

    \begin{align}
        \esp{Z_4}
        & = w_tt + w_1\\
        & = 1 - w_0 + w_t(t-1)\\
        & = 1 - w_0 - (1-c-w_0)\\
        & = c\\
        \var{Z_4}
        & = w_0c^2 + w_t(t - c)^2 + w_1(1-c)^2\\
        & = (1-c)^2 + w_0(2c-1) - w_t(t-1)(t+1-2c)\\
        & = (1-c)^2 + w_0(2c-1) - (1-c-w_0)(t+1-2c)\\
        & = v - \svmin - (1-c)(c-t)\\
        & = v\\
        R_{Z_4}
        & = \unorm\esp{(\1{Z_4\geq t} - 1)(Z_4 - t)}\\
        & = \unorm\esp{\1{Z_4 < t}(t - Z_4)}\\
        & = w_0t\\
        & = v - \svmin\\
        & = L_{Z_4}
    \end{align}

    Note that in the cases where $v < \svmin$, the Bhatia-Davis inequality cannot be saturated and an extra dirac $\delta_c$ is added in $c$.

    For each of the above distributions, we have: $\esp{Z_i} = c$, $\var{Z_i} = v$ and the regret equals the lower bound $R_{Z_i} = L_{Z_i}$.

    Now we link $Z$ back to the true probabilities $\PPG{\hstar(X)}{\hh(X) = p}$. Let $p \in [0, 1]$ such that $\hh$ has at least 3 antecedant values that we note $x_1$, $x_2$ and $x_3$. We can thus create the joint distribution $\mathbb{P}_{(X, Y)}$ as a discrete distribution
    with $\mathbb{P}(x_i) = z_i$ and $\PPG{Y=1}{X = x_i} = w_i$ for all $i \in \{1, 2, 3\}$, where triplets $(w_1, w_2, w_3)$ are taken from each above cases depending on the values of $c$ and $v$, and $(z_1, z_2, z_3)$ their associated positions. Then, the distribution $\PPG{\hstar(X)}{\hh(X) = p}$ has mean $c$ and variance $v$ and $\rglh(p) = \lglh(p)$.
\end{proof}

\subsection{Grouping regret upper bound}
\label{proof:regret:ub}

\thmRegretUB*
\begin{proof}[Proof of \autoref{thm:regret:ub}]
    Let $p \in \supp{\hh(X)}$.

    Suppose $0 < \PPG{\hstar(X) \geq t\sstar}{\hh(X) = p} < 1$.

    Define $w(p) \triangleq \PPG{\hstar(X) \geq t\sstar}{\hh(X)=p}$ and $c^+(p) \triangleq \espk{\hstar(X)}{\hstar(X) \geq t\sstar, \hh(X)=p}$. \autoref{lem:regret:reform} gives:
    \begin{equation}
        \label{eq:proof:ub}
        \rglh(p) = \unorm\pr{w(p)(c^+(p) - t\sstar) - \1{\ch(p) \geq t\sstar}(\ch(p) - t\sstar)}
    \end{equation}

    We apply the lower bound of \autoref{lem:var} to $Z \triangleq \hstar(X) | \hh(X)$, which gives:
    \begin{align}
        \sqrt{\vfhp}
        & \geq \sqrt{w(p)(1-w(p))(c^+(p) - c^-(p))^2} & \text{\autoref{eq:var:lb}}\\
        & = \sqrt{w(p)(1-w(p))}(c^+(p) - c^-(p)) & c^+(p) \geq c^-(p)\\
        & = \sqrt{\frac{w(p)}{1-w(p)}}(c^+(p) - \ch(p)) & \text{Using \autoref{eq:mu:diff2}}
    \end{align}
    Hence:
    \begin{align}
        w(p)(c^+(p) - t\sstar) & \leq \sqrt{w(p)(1-w(p))\vfhp} + w(p)(\ch(p) - t\sstar)&
    \end{align}
    Using \autoref{eq:proof:ub}, we have:
    \begin{align}
        \label{eq:proof:ub2}
        \rglh(p) \leq \unorm\pr{\sqrt{w(p)(1-w(p))\vfhp} + (w(p) - \1{\ch(p) \geq t\sstar})(\ch(p) - t\sstar)}
    \end{align}

    We showed \autoref{eq:proof:ub2} for all $w(p) \in (0, 1)$. It still holds for $w(p) \in \{0, 1\}$. Indeed, \autoref{lem:regret:reform} shows that $\rglh(p) = 0$ in this case. Also, when $w(p) \in \{0, 1\}$, then $w(p) - \1{\ch(p) \geq t\sstar} = 0$. Hence the right-hand side of \autoref{eq:proof:ub2} is also 0 for $w(p) \in \{0, 1\}$, hence equal to $\rglh(p)$.

    Since \autoref{eq:proof:ub2} holds for all $w(p) \in [0, 1]$, we now find the worst case. That is, the $w(p) \in [0, 1]$ that maximizes the right-hand side of \autoref{eq:proof:ub2}.

    First suppose that $\ch(p) \geq t\sstar$.
    Let $g_{a,b}(w) \triangleq \sqrt{w(1-w)b} + (w - 1)a$ with $b \geq 0$. Note that $\text{(\ref{eq:proof:ub2})} = \unorm\,g_{a,b}(w(p))$ with $a = \ch(p) - t\sstar$ and $b = \vfhp$. The maximum of $g_{a,b}$ is reached at $w = \frac{1}{2}(1 - \sqrt{\frac{a^2}{a^2 + b}})$ with value $\frac{1}{2}(\sqrt{a^2 + b} - a)$. Hence the maximum of (\ref{eq:proof:ub2}) is:
    \begin{equation}
        \frac{\unorm}{2}(\sqrt{\vfhp + (\ch(p) - t\sstar)^2} - (\ch(p) - t\sstar)).
        \label{eq:app:max}
    \end{equation}

    Now suppose that $\ch(p) < t\sstar$.
    Similarly, let $g_{a,b}(w) \triangleq \sqrt{w(1-w)b} + wa$ with $b \geq 0$. Note that $\text{(\ref{eq:proof:ub2})} = \unorm\,g_{a,b}(w(p))$ with $a = \ch(p) - t\sstar$ and $b = \vfhp$. The maximum of $g_{a,b}$ is reached at $w = \frac{1}{2}(1 - \sqrt{\frac{a^2}{a^2 + b}})$ with value $\frac{1}{2}(\sqrt{a^2 + b} + a)$. Hence the maximum of (\ref{eq:proof:ub2}) is:
    \begin{equation}
        \frac{\unorm}{2}(\sqrt{\vfhp + (\ch(p) - t\sstar)^2} - (t\sstar - \ch(p))).
    \end{equation}

    In both cases, the maximum of (\ref{eq:proof:ub2}) is:
    \begin{equation}
        \frac{\unorm}{2}(\sqrt{\vfhp + (\ch(p) - t\sstar)^2} - |\ch(p) - t\sstar|).
    \end{equation}
    which concludes the proof.

    \paragraph{Tightness of the regret upper bound $\uglh$.} Following the proof of the regret upper bound, we see that the bound is obtained using the lower bound of the lemma on the variance (\autoref{lem:var}). The proof of \autoref{lem:var} gives an idea on how to achieve the equality. Taking a distribution where $v^+ = 0$ and $v^- = 0$ is a good candidate. That is, two diracs in $[0, 1]$ with appropriate weights.

    Below we explicit distributions achieving the equality between the regret $\rglh$ and the upper bound $\rglh$. For convenience, we work with a random variable $Z$ valued in $[0,1]$ which we will later link to the true probabilities $\mathbb{P}_{\hstar|h = p}$.

    Suppose $t\sstar \in (0, 1)$. Let $c \in [0, 1]$ and $v \in [0, c(1 - c)]$. The Bhatia-Davis inequality shows that these represent the sets of admissible values of mean and variances of a random variable valued in $[0,1]$. Note $R_Z \triangleq \unorm\esp{(\1{Z\geq t\sstar} - \1{\esp{Z}\geq t\sstar})(Z - t\sstar)}$ and $U_Z \triangleq \tfrac{1}{2}\unorm\!(\sqrt{\var{Z} + (\esp{Z}-t^{\star})^2} - |\esp{Z} - t^{\star}|)$.
    $R_Z$ represents the regret $\rgl$ (\autoref{thm:rgl:expression})
    and $U_Z$ its upper bound $\ugl$ (\autoref{thm:regret:ub}).

    We now show that there exist distributions $\mathbb{P}_Z$ of $Z$ such that $\esp{Z} = c$, $\var{Z} = v$ and $R_Z = U_Z$. Depending of the value of $c$, we explicit 2 distributions of $Z$ that achieve the equality.

    Define $w = \tfrac{1}{2}(1 - \sqrt{\frac{(c - t\sstar)^2}{(c-t\sstar)^2 + v}})$ and $\mathbb{P}_Z = (1-w)\delta_a + w \delta_b$ with $a, b \in [0, 1]$. Suppose $t\sstar = \frac{1}{2}$.

    \paragraph{Case 1: $c < t\sstar$.} Take $a = c - \sqrt{v\tfrac{w}{1-w}}$ and $b = c + \sqrt{v\tfrac{1-w}{w}}$. Both $a$ and $b$ are in $[0, 1]$ when $t\sstar = \frac{1}{2}$.
    \begin{align}
        \esp{Z_1}
        & = (1-w)a + wb\\
        & = c - \sqrt{vw(1-w)} + \sqrt{vw(1-w)}\\
        & = c\\
        \var{Z_1}
        & = (1-w)(a-c)^2 + w(b-c)^2\\
        & = vw + v(1-w)\\
        & = v\\
        R_{Z_1}
        & = \unorm\esp{\1{Z_1\geq t\sstar}(Z_1 - t\sstar)}\\
        & = w(c-t\sstar+\sqrt{v\tfrac{1-w}{w}}) &b \geq t\sstar \text{ and } a \leq c \leq t\sstar\\
        & = \sqrt{vw(1-w)} + w(c-t\sstar)\\
        & = \frac{1}{2}(\sqrt{v + (c-t\sstar)^2} - (c-t\sstar)) + (c-t\sstar)&\text{\autoref{eq:app:max}}\\
        & = \frac{1}{2}(\sqrt{v + (c-t\sstar)^2} - (t\sstar-c))\\
        & = U_{Z_1}
    \end{align}

    \paragraph{Case 2: $c \geq t\sstar$.}Take $a = c - \sqrt{v\tfrac{1-w}{w}}$ and $b = c + \sqrt{v\tfrac{w}{1-w}}$. Both $a$ and $b$ are in $[0, 1]$ when $t\sstar = \frac{1}{2}$.

    \begin{align}
        \esp{Z_2}
        & = (1-w)a + wb\\
        & = c + \sqrt{vw(1-w)} - \sqrt{vw(1-w)}\\
        & = c\\
        \var{Z_2}
        & = (1-w)(a-c)^2 + w(b-c)^2\\
        & = vw + v(1-w)\\
        & = v\\
    \end{align}
    \begin{align}
    R_{Z_2}
        & = \unorm\esp{\1{Z_2 < t\sstar}(t\sstar - Z_2)}\\
        & = w(t\sstar - c + \sqrt{v\tfrac{1-w}{w}})& b \geq c \geq t\sstar \text{ and } a \leq c \leq t\sstar\\
        & = \sqrt{vw(1-w)} + w(t\sstar - c)\\
        & = \frac{1}{2}(\sqrt{v + (c-t\sstar)^2} - (t\sstar-c)) + (t\sstar-c)&\text{\autoref{eq:app:max}}\\
        & = \frac{1}{2}(\sqrt{v + (c-t\sstar)^2} - (t\sstar-c))\\
        & = U_{Z_2}
    \end{align}

    For each of the above distributions, we have: $\esp{Z_i} = c$, $\var{Z_i} = v$ and the regret equals the lower bound $R_{Z_i} = U_{Z_i}$.

    Now we link $Z$ back to the true probabilities $\PPG{\hstar(X)}{\hh(X) = p}$. Let $p \in [0, 1]$ such that $\hh$ has at least 2 antecedant values that we note $x_1$ and $x_2$. We can thus create the joint distribution $\mathbb{P}_{(X, Y)}$ as a discrete distribution
    with $\mathbb{P}(x_1) = a$, $\mathbb{P}(x_2) = b$, $\PPG{Y=1}{X = x_1} = 1-w$ and $\PPG{Y=1}{X = x_2} = w$ Then, the distribution $\PPG{\hstar(X)}{\hh(X) = p}$ has mean $c$ and variance $v$ and $\rglh(p) = \uglh(p)$.

    If $p \in [0, 1]$ is such that $\hh$ has only one antecedant value, then there is no grouping loss in the level set $\hh = p$, \ie{} $\GL(p) = 0$. Hence, $\rglh(p) = 0 = \uglh(p)$.

\end{proof}

\subsection{GLAR -- Grouping Loss Adaptative Recalibration}
\label{sec:app:glar}

\begin{restatable}[GLAR reduces grouping loss, \ref{sec:app:glar}]{proposition}{propglar}
    \label{prop:glar}
    Let a function $\hh : \Xc \to [0, 1]$ and a partition $\Pc : \Xc \to E$ of the feature space $\Xc$. Then the GLAR-estimator $\hhp$ defined in \autoref{eq:recal_ours} satisfies:
    \begin{align}
        \label{eq:glar}
        \espk{Y}{\hhp(X) = p} & = p & \text{$\hhp$ is calibrated}\\
        \GL(\hhp) \leq \GL(\hh)&   & \text{$\hhp$ has lower $\GL$ than $\hh$}
    \end{align}
    The grouping loss is reduced by: $\quad \GL(\hh) - \GL(\hhp) = \esp{\vark{\hhp(X)}{\hh(X)}}$.\\
    The remaining grouping loss is: $\quad \GL(\hhp) = \esp{\vark{\hstar(X)}{\hh(X), \Pc(X)}}$.
\end{restatable}

\begin{proof}
    Let $\Pc$ a partition of the feature space.
    Define $P = \Pc(X)$, $\S = \hh(X)$, $\S_P = \hhp(X)$ and $\GL(\S) = \esp{\vark{\Q}{\S}}$. Using definition of $\S_P$, we have: $\S_P = \espk{Y}{\S, P} = \espk{\espk{\Q}{X}}{\S,P} = \espk{\Q}{\S,P}$

    Following \citet[Theorem 4.1]{Perez-Lebel2023}, we have $\GL(\S) = \GL_{\mathrm{explained}}(\S) + \GL_{\mathrm{residual}}(\S)$ with $\GL_{\mathrm{explained}}(\S) = \esp{\vark{\espk{\Q}{\S,P}}{\S}}$ and $\GL_{\mathrm{residual}}(\S) = \esp{\vark{\Q}{\S,P}}$.

    \begin{align}
        & \GL_{\mathrm{explained}}(\S) - \GL_{\mathrm{explained}}(\S_P)\\
        & = \esp{\vark{\espk{\Q}{\S,P}}{\S}} - \esp{\vark{\espk{\Q}{\S,P}}{\S_P}}&\\
        & = \esp{\vark{\espk{\Q}{\S,P}}{\S}} - \esp{\vark{\S_P}{\S_P}}& \S_P = \espk{\Q}{\S,P}\\
        & = \esp{\vark{\espk{\Q}{\S,P}}{\S}} & \vark{\S_P}{\S_P} = 0\\
        & = \esp{\vark{\S_P}{\S}} &
    \end{align}

    We apply the law of total variance on the residual term by conditioning on $\S$:
    \begin{align}
        \GL_{\mathrm{residual}}(\S_P)
        & = \esp{\vark{\Q}{\S_P,P}}\\
        & = \esp{\vark{\espk{\Q}{\S,\S_P,P}}{\S_P,P} + \espk{\vark{\Q}{\S,\S_P,P}}{\S_P,P}}\\
        & = \esp{\vark{\espk{\Q}{\S,P}}{\S_P,P} + \espk{\vark{\Q}{\S,P}}{\S_P,P}}\\
        & = \esp{\vark{\S_P}{\S_P,P} + \espk{\vark{\Q}{\S,P}}{\S_P,P}}\\
        & = \esp{\espk{\vark{\Q}{\S,P}}{\S_P,P}}\\
        & = \esp{\vark{\Q}{\S,P}}\\
        & = \GL_{\mathrm{residual}}(\S)
    \end{align}

    Hence:
    \begin{align}
        & \GL(\S) - \GL(\S_P)\\
        & = \GL_{\mathrm{explained}}(\S) - \GL_{\mathrm{explained}}(\S_P) +  \GL_{\mathrm{residual}}(\S) - \GL_{\mathrm{residual}}(\S_P)\\
        & = \esp{\vark{\espk{\Q}{\S,P}}{\S}}
    \end{align}
\end{proof}

\begin{restatable}[Hierarchy of partition-based decision rules]{lemma}{lemhierarchy}
    \label{lem:hierarchy}
    Let $\Pc_1, \Pc_2 : \Xc \to E$ be two partitions of the feature space such that $\Pc_1(X)$ is $\Pc_2(X)$-measurable.
    Let $g_\Pc : x \mapsto \espk{Y}{\Pc(X)=\Pc(x)}$. Then:
    \begin{align}
        \EU(\delta_{g_{\Pc_1}, t}) \leq \EU(\delta_{g_{\Pc_2}, t\sstar}) && \forall t \in [0, 1]
    \end{align}
\end{restatable}
\autoref{lem:hierarchy} shows that the finer the partition the better the expected utility. Since all of the original, recalibrated and GLAR-corrected can be seen as instance of GLAR with different granularity of partition (the trivial partition, the partition on $\hh$, and the partition on $(\hh, \Pc)$ respectively), this enables to conclude on a hierarchy between these decision rules (\autoref{prop:hierarchy}).

\begin{proof}
    Let $\Pc_1, \Pc_2 : \Xc \to \R^d$ be two partitions of the feature space such that $\Pc_1(X)$ is $\Pc_2(X)$-measurable.
    Let $g_\Pc : x \mapsto \espk{Y}{\Pc(X)=\Pc(x)}$. Let $G_1 \triangleq g_{\Pc_1}(X)$, $G_2 \triangleq g_{\Pc_2}(X)$, $P_1 \triangleq \Pc_1(X)$ and $P_2 \triangleq \Pc_2(X)$ the associated random variables.

    \begin{align*}
        & \EU(\delta_{g_{\Pc_2}, \t}) - \EU(\delta_{g_{\Pc_1}, \t})\\
        & = \esp[X]{(\1{G_2 \geq \t} - \1{G_1 \geq \t})(\tfrac{\Q}{\t} - 1)} & \text{\autoref{cor:expected-utility:diff}}\\
        & = \esp[X]{\espk{(\1{G_2 \geq \t} - \1{G_1 \geq \t})(\tfrac{\Q}{\t} - 1)}{\Pc_2}} & \text{Total expectation}\\
        & = \esp[X]{(\1{G_2 \geq \t} - \1{G_1 \geq \t})\espk{\tfrac{\Q}{\t} - 1}{\Pc_2}} & \text{$P_1$ is $P_2$-measurable}\\
        & = \esp[X]{(\1{G_2 \geq \t} - \1{G_1 \geq \t})\pr{\tfrac{\espk{\Q}{\Pc_2}}{\t} - 1}} & \text{Linearity}\\
        & = \esp[X]{(\1{G_2 \geq \t} - \1{G_1 \geq \t})\pr{\tfrac{\espk{\espk{Y}{X}}{\Pc_2}}{\t} - 1}} & \text{Definition of $\Q$}\\
        & = \esp[X]{(\1{G_2 \geq \t} - \1{G_1 \geq \t})\pr{\tfrac{\espk{Y}{\Pc_2}}{\t} - 1}} & \text{Total expectation}\\
        & = \esp[X]{(\1{G_2 \geq \t} - \1{G_1 \geq \t})\pr{\tfrac{G_2}{\t} - 1}} & \text{Definition of $G_2$}\\
        & = \esp[X]{\1{G_2 \geq \t}\1{G_1 < \t}\pr{\tfrac{G_2}{\t} - 1} + \1{G_2 < \t}\1{G_1 \geq \t}\pr{1 - \tfrac{G_2}{\t}}} & \\
        & \geq 0
    \end{align*}

\end{proof}

\begin{restatable}[Hierarchy of decision rules, \ref{sec:app:glar}]{proposition}{prophierarchy}
    \label{prop:hierarchy}
    Let $\Pc_1, \Pc_2 : \Xc \to \R$ be two partitions of the feature space such that $\Pc_1$ is constant on regions of $\Pc_2$ (\ie{} $\Pc_1$ is a function of $\Pc_2$).
    Then, for all $\delta_{\hh,t} \in \mathcal{D}_\hh$:
    \begin{equation*}
        \EU(\delta_{\hh,t})
        \leq
        \EU(\delta_{\cch, t\sstar})
        \leq
        \EU(\delta_{\hh_{\Pc_1}, t\sstar})
        \leq
        \EU(\delta_{\hh_{\Pc_2}, t\sstar})
        \leq
        \EU(\delta_{\hstar, t\sstar})
    \end{equation*}
\end{restatable}
\begin{proof}
    Let $\Pc_1, \Pc_2 : \Xc \to \R^d$ be two partitions of the feature space such that $\Pc_1(X)$ is $\Pc_2(X)$-measurable.

    Define $P_0 = 0$, $P_c = \hh(X)$, $P_1 = (P_c, \Pc_1(X))$ and $P_2 = (P_c, \Pc_2(X))$. We have: $P_0$ is $P_c$-measurable, $P_c$ is $P_1$-measurable, and $P_1$ is $P_2$-measurable. Hence, \autoref{lem:hierarchy} concludes the proof.

\end{proof}

\paragraph{A threshold view}
A duality between probability correction and adaptative thresholding also exists for GLAR. Instead of correcting the estimated probabilities on groups with $\hhp$, GLAR can also be viewed as using the original estimates $\hh$ with an adaptative threshold $t_\Pc$ instead (\autoref{def:glat}): $\delta_{\hh, t_\Pc}$ and $\delta_{\hhp, \t}$ have same regret.

Instead of correcting the estimated probabilities on groups with $\hhp$, this procedure can be viewed as using the original estimates $\hh$ with an adaptative threshold $t_\Pc$ instead (\autoref{def:glat}):

\begin{definition}[Grouping loss adaptative threshold GLAT]
    \label{def:glat}
    \begin{equation}
        \label{eq:gl-threshold}
        \begin{aligned}
        t_{\mathcal{P}} \colon x \mapsto t\sstar - (\hh_{\mathcal{P}}(x) - \hh(x))
        \end{aligned}
    \end{equation}
\end{definition}
$\delta_{\hh, t_\Pc}$ and $\delta_{\hhp, \t}$ have same regret: $\hh(x) \geq t_{\Pc} \Leftrightarrow \hh_\Pc(x) \geq t\sstar$.

\clearpage
\section{Experiments}
\label{sec:app:xp}

\subsection{Model definitions}
\input{img/decision/xp/test_common/translate_models/model_translation.tex}

\subsection{Dataset definitions}
\setlength{\tabcolsep}{3pt}
\input{img/decision/xp/test_common/translate_tables/ds_translation.tex}

\subsection{Experimental Details}
\label{sec:details}

\paragraph{Data}
For each of the pre-trained models for hate-speech detection \autoref{tab:model_translation}, we use the embedding space of the pernultimate layers as input space $\Xc$ for estimating the grouping loss. We note $f : \Xc \to [0, 1]$ the function mapping the embedding space of the model to the hate-speech probability esitmate. We forward the datasets listed in \autoref{tab:ds_translation} through each of the models. We store the ground truth label $Y$, which equals to 1 if the text is hate speech and 0 otherwise, the embedding $X$, and the model's predicted probability $\hh(X)$ that the text is hate-speech. Then we work with each triplet $(X, Y, f(X))$ to investigate the effect of post-training on decision-making. The data is split into a train set on which post-training methods are fit, and a test set on which the post-training methods are applied and the expected utility and regrets are estimated.

\paragraph{Utility} We define binary utility matrices $U \in \R^{2 \times 2}$ of the shape:
\begin{equation}
    U \triangleq
    \begin{bmatrix}
        1 & 0\\
        0 & U_{11}
    \end{bmatrix}
\end{equation}
with $U_{11} \in \R^+$. We select values of $U_{11}$ so that the optimal threshold $t\sstar$ takes values $[0.01, 0.025, 0.05, 0.1, 0.25, 0.5, 0.75, 0.9, 0.95, 0.975, 0.99]$. Note that the link between $U_{11}$ and $t\sstar$ is given by $U_{11} = \tfrac{1}{t\sstar} - 1$ (\autoref{eq:optimal_classifier}). Each utilty matrix $U$ represents a different decision-making problem. Following the decision theory result of \citet{Elkan2001}, the optimal decision rule is $\delta_{\hstar, t\sstar}$. We thus use $\delta_{\hh, t\sstar}$ as decision rule from the probability esimates given by $f$.

The expected utility of decision rule $\delta_{\hh, t\sstar}$ is given by $\EU(\delta_{\hh, t\sstar}) = \esp{U[\delta_{\hh, t\sstar}(X), Y]}$. We estimte the expected utility using the emprirical utility $\widehat{\EU} \triangleq \frac{1}{n}\sum_i U[\delta_{\hh, t\sstar}(X_i), Y_i]$ where $(X_i, Y_i)$ are the samples of $(X, Y)$ and $n$ is the number of samples.

\paragraph{Regret} We estimate the regrets using the expression given in \autoref{def:intro:estimators}. This necessitates estimates of the calibration curve $\ch$ and the grouping loss. To estimate $\ch$ we using an histogram binning with 15 equal-mass bins. The esitmation of the grouping loss is detailed in the next paragarph.

\paragraph{Grouping Loss estimation}
For the estimation of the grouping loss we follow the estimation procedure given by \citet{Perez-Lebel2023}. We use 15 equal-mass bins on the probability space $[0, 1]$. For partitioning the embedding space $\Xc$, we use a decision tree on the pair $(X, Y)$ constrained to create at most 5 regions per bin. The partition $\Pc$ is derived from the fitted decision tree using the leaves of the tree: each leave forms a part of the partition. To fit the decision tree we use a train-test split strategy, as recommended by \citet{Perez-Lebel2023}, akin to a honest tree. The decision tree is fitted on the first part of the train set and the local averages and grouping loss are estimated on the second part of the train set. This avoids overfitting.

\paragraph{Recalibration} We use recalibration methods to correct the estimated probabilities. Below are listed the methods used and implementation details. Recalibration methods are fitted on the train set and applied on the test set.
\begin{itemize}
    \item \emph{Isotonic Regression}. We use scikit-learn's implementation \texttt{IsotonicRegression} from \texttt{sklearn.isotonic} and fit it on (S, Y). This method has no hyperparameters.
    \item \emph{Platt Scaling}. We use scikit-learn's implementation \texttt{\_SigmoidCalibration} from \texttt{sklearn.calibration} and fit it on (S, Y). This method has no hyperparameters.
    \item \emph{Histogram Binning}. We use the implementation of \texttt{HistogramCalibrator} from the \texttt{calibration} python package. We use 15 equal-mass bins.
    \item \emph{Scaling-Binning}. We use the implementation of \texttt{PlattBinnerCalibrator} from the \texttt{calibration} python package. We use 15 equal-mass bins.
    \item \emph{Meta-Cal}. We use the implementation of \texttt{MetaCalMisCoverage} from the \texttt{metacal} python package. We use a miscoverage of 0.05 based on the default value provided in the package's example \texttt{metacal/examples/test\_metacal.py}.
\end{itemize}

\paragraph{Post-training} Besides recalibration, we also investigate post-training methods. Below are listed the methods used and implementation details. Post-training methods are fitted on the train set and applied on the test set.
\begin{itemize}
    \item \emph{Fine-tuning}. Since the link between each hate-speech model's embedding space and the output space is a sigmoid, we use scikit-learn to learn the sigmoid function with \texttt{LogisticRegression} from \texttt{sklearn.linear\_model} with default parameters. In this form of finetuning, the initial model weight is not used. This contrasts with finetuning using pytorch's optimizers starting from the pre-trained model's weights.
    \item \emph{Stacking}. We use \texttt{RandomForestClassifier} and \texttt{HistGradientBoostingClassifier} implementation from scikit-learn with default parameters for the stacked models. The stack model is fit on the augmented space $(X, \hh(X))$ as input, that is $((X, \hh(X)), Y)$.
    \item \emph{GLAR}. We implemented GLAR as described in
    the next paragraph.
\end{itemize}

\paragraph{Metrics} The performance metrics used in the experiments are:
\begin{itemize}
    \item \emph{Brier score} defined as $\mathrm{Brier} \triangleq \esp{(\hh(X) - Y)^2}$.
    \item \emph{Expected Calibration Error (ECE)} defined as $\mathrm{ECE} \triangleq \esp{|\espk{Y}{\hh(X)} - \hh(X)|}$.
    \item \emph{Calibration Loss (CL)} defined as $\mathrm{CL} \triangleq \esp{(\espk{Y}{\hh(X)} - \hh(X))^2}$.
    \item \emph{Root Mean Square Calibration Error (RMSCE)} defined as $\mathrm{RMSCE} \triangleq \sqrt{\esp{(\espk{Y}{\hh(X)} - \hh(X))^2}}$.
    \item \emph{Maximum Calibration Error (MCE)} defined as $\mathrm{MCE} \triangleq \esp{\max |\espk{Y}{\hh(X)} - \hh(X)|}$.
\end{itemize}

\paragraph{Compute setting}
All experiments ran on a single compute node of 256 CPUs.

\paragraph{Code repository}
The code used for the experiments is available at \url{https://github.com/aperezlebel/decision_suboptimal_classifiers}.

\paragraph{GLAR}
GLAR builds on the grouping loss estimator proposed by \citet{Perez-Lebel2023}. GLAR uses the partitioning strategy based on a decision tree detailed in \autoref{sec:details}. Once identified the partition $\Pc$ of the input space, GLAR estimates the local averages of the probabilities $\hh(X)$ on each region of the partition. The GLAR estimator $\hhp$ is then the function mapping each input $x$ to the average of the probabilities on the region of the partition to which $x$ belongs. It is applied on the test set to correct the estimated probabilities similarly to the other post-training methods.

For each model $f$, we split the probability range $[0, 1]$ into 15 equal-mass bins. In each bin, we use a decision tree constrained to form at most 5 leaves. To regularize the correction by GLAR, we use a threshold strategy. We set a threshold $r \in \R$. GLAR applies the correction if the overall grouping regret $\rglhhat$ is larger than $r$ \emph{and} only corrects the bins for which the conditional grouping regret $\rglhhat(p)$ is larger than $r$. We set $r = 0.02$. Otherwise, GLAR applies an isotonic correction. This leverages the grouping regret estimator and avoid correcting the probabilities beyond calibration when it is not needed.

\clearpage
\subsection{Post-training gain results}
\label{sec:app:gains}

\begin{figure}[ht]
    \centering
    \includegraphics[width=0.33\linewidth]{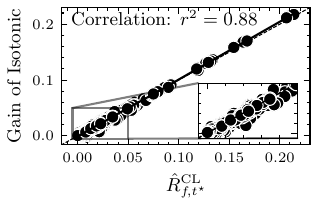}%
    \includegraphics[width=0.33\linewidth]{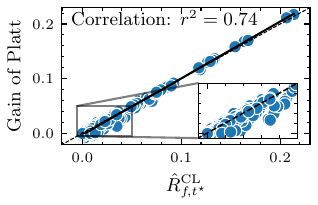}%
    \includegraphics[width=0.33\linewidth]{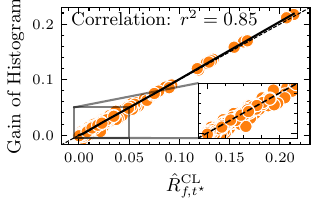}\\
    \includegraphics[width=0.33\linewidth]{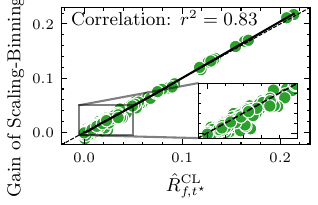}%
    \includegraphics[width=0.33\linewidth]{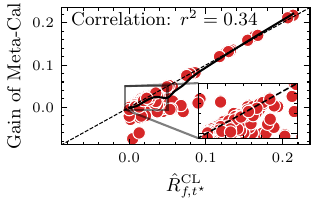}%
    \caption{\textbf{Recalibration gain vs $\rclhtstarhat$.} Gain of each recalibration method vs $\rclhtstarhat$.}
    \label{fig:app:gain:rcl}
\end{figure}

\begin{figure}[ht]
    \centering
    \includegraphics[width=0.45\linewidth]{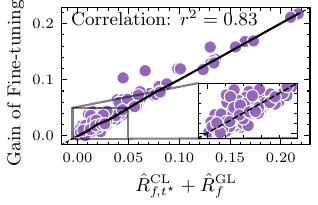}%
    \includegraphics[width=0.45\linewidth]{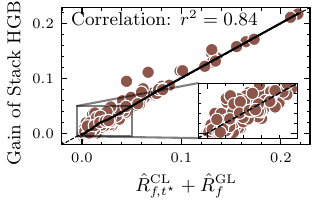}\\
    \includegraphics[width=0.45\linewidth]{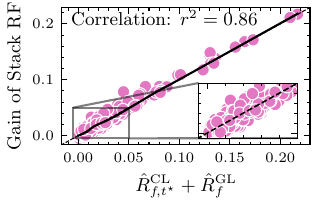}%
    \includegraphics[width=0.45\linewidth]{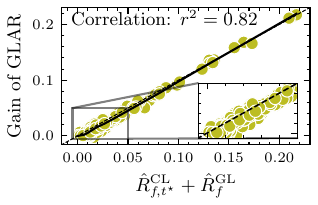}
    \caption{\textbf{Post-training gain vs $\rclhtstarhat + \rglhhat$.} Gain of non-recalibration methods vs $\rclhtstarhat + \rglhhat$.}
    \label{fig:app:gain:rclglmean}
\end{figure}

\begin{figure}[ht]
    \centering
    \includegraphics[width=0.33\linewidth]{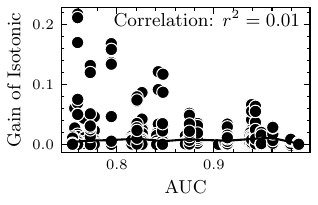}%
    \includegraphics[width=0.33\linewidth]{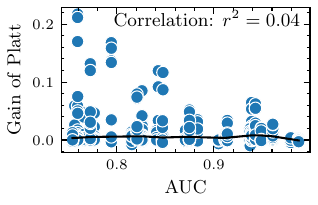}%
    \includegraphics[width=0.33\linewidth]{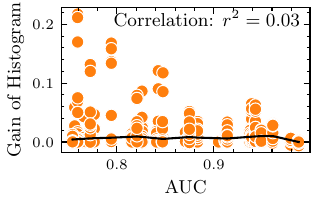}\\
    \includegraphics[width=0.33\linewidth]{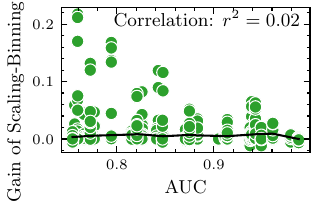}%
    \includegraphics[width=0.33\linewidth]{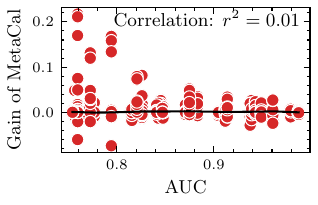}%
    \includegraphics[width=0.33\linewidth]{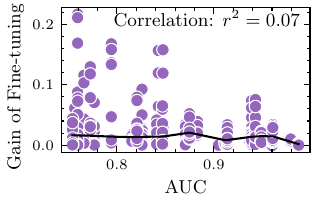}\\
    \includegraphics[width=0.33\linewidth]{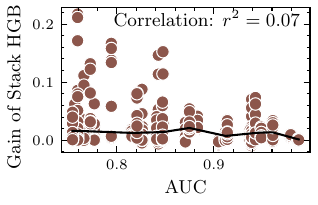}%
    \includegraphics[width=0.33\linewidth]{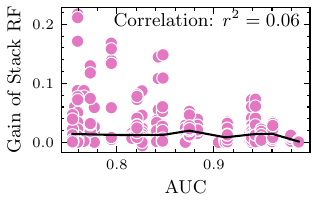}%
    \includegraphics[width=0.33\linewidth]{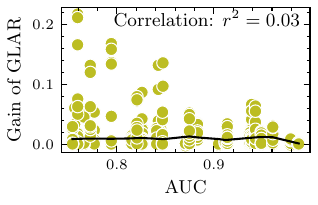}
    \caption{\textbf{Gain vs AUC.} Gain of each post-training method vs the AUC of the model.}
    \label{fig:app:gain:auc}
\end{figure}
\begin{figure}[ht]
    \centering
    \includegraphics[width=0.33\linewidth]{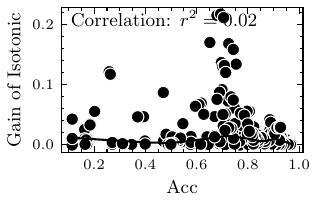}%
    \includegraphics[width=0.33\linewidth]{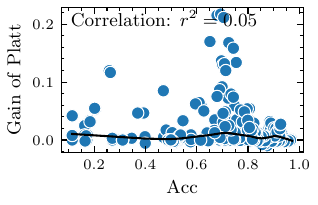}%
    \includegraphics[width=0.33\linewidth]{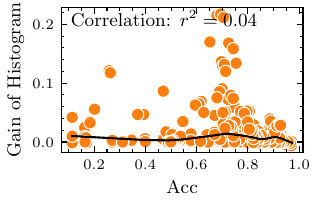}\\
    \includegraphics[width=0.33\linewidth]{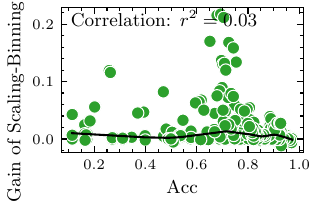}%
    \includegraphics[width=0.33\linewidth]{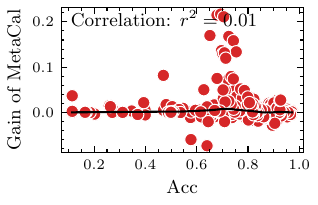}%
    \includegraphics[width=0.33\linewidth]{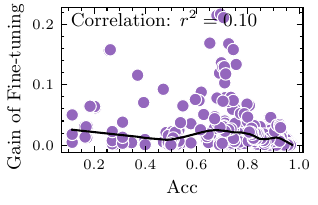}\\
    \includegraphics[width=0.33\linewidth]{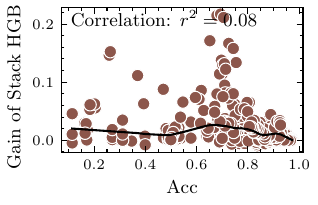}%
    \includegraphics[width=0.33\linewidth]{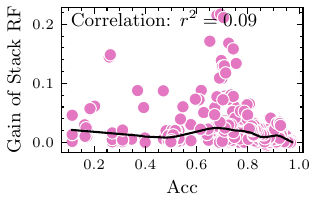}%
    \includegraphics[width=0.33\linewidth]{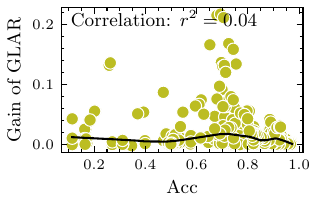}
    \caption{\textbf{Gain vs Accuracy.} Gain of each post-training method vs the Accuracy of the model.}
    \label{fig:app:gain:acc}
\end{figure}
\begin{figure}[ht]
    \centering
    \includegraphics[width=0.33\linewidth]{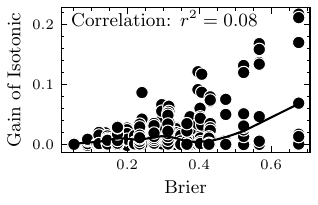}%
    \includegraphics[width=0.33\linewidth]{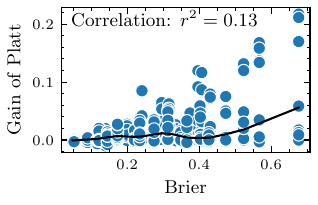}%
    \includegraphics[width=0.33\linewidth]{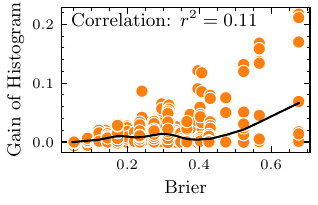}\\
    \includegraphics[width=0.33\linewidth]{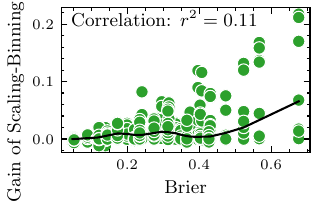}%
    \includegraphics[width=0.33\linewidth]{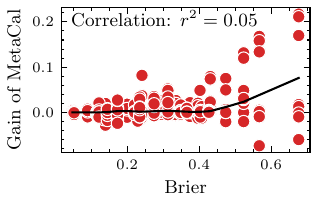}%
    \includegraphics[width=0.33\linewidth]{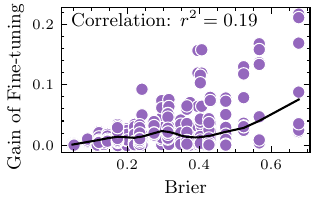}\\
    \includegraphics[width=0.33\linewidth]{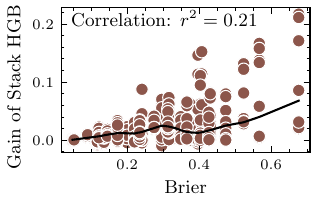}%
    \includegraphics[width=0.33\linewidth]{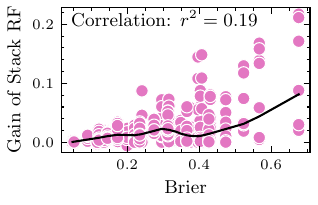}%
    \includegraphics[width=0.33\linewidth]{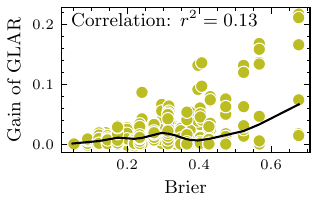}
    \caption{\textbf{Gain vs Brier.} Gain of each post-training method vs the Brier score of the model.}
    \label{fig:app:gain:brier}
\end{figure}

\begin{figure}[ht]
    \centering
    \includegraphics[width=0.33\linewidth]{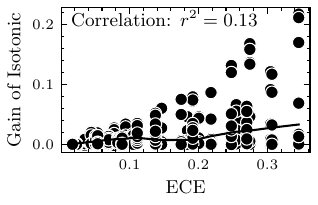}%
    \includegraphics[width=0.33\linewidth]{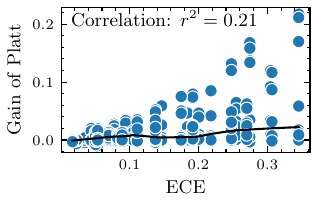}%
    \includegraphics[width=0.33\linewidth]{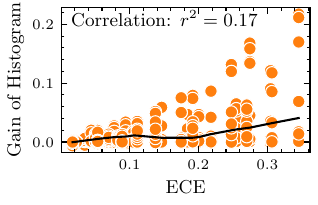}\\
    \includegraphics[width=0.33\linewidth]{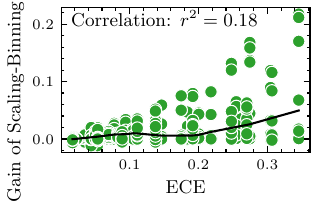}%
    \includegraphics[width=0.33\linewidth]{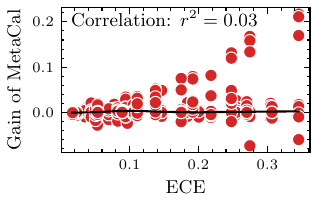}%
    \includegraphics[width=0.33\linewidth]{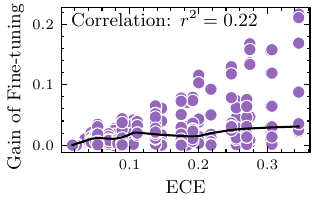}\\
    \includegraphics[width=0.33\linewidth]{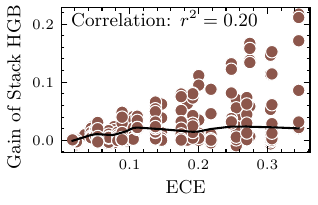}%
    \includegraphics[width=0.33\linewidth]{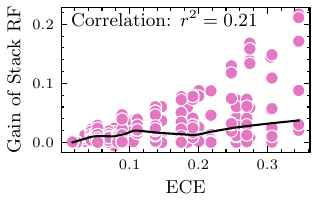}%
    \includegraphics[width=0.33\linewidth]{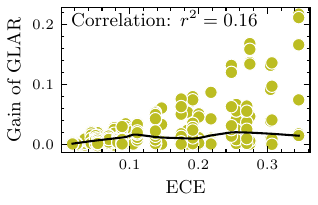}
    \caption{\textbf{Gain vs $\ece$.} Gain of each post-training method vs the $\ece$ of the model.}
    \label{fig:app:gain:ece}
\end{figure}

\begin{figure}[ht]
    \centering
    \includegraphics[width=0.33\linewidth]{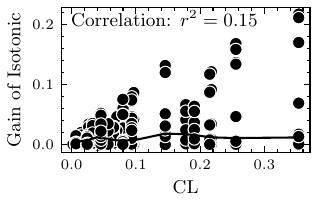}%
    \includegraphics[width=0.33\linewidth]{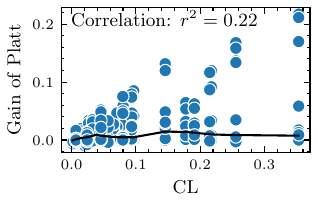}%
    \includegraphics[width=0.33\linewidth]{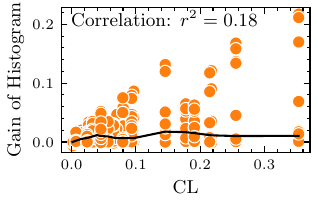}\\
    \includegraphics[width=0.33\linewidth]{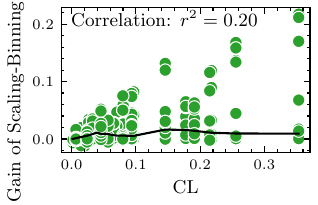}%
    \includegraphics[width=0.33\linewidth]{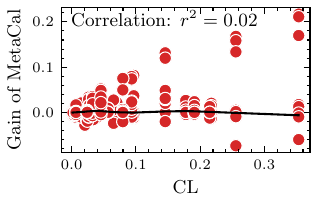}%
    \includegraphics[width=0.33\linewidth]{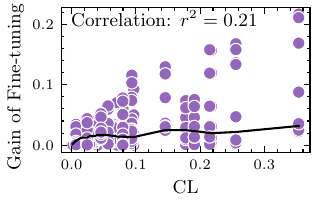}\\
    \includegraphics[width=0.33\linewidth]{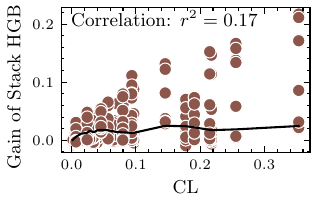}%
    \includegraphics[width=0.33\linewidth]{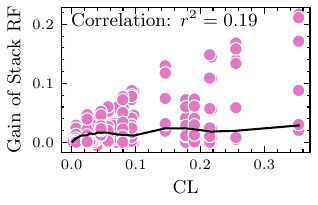}%
    \includegraphics[width=0.33\linewidth]{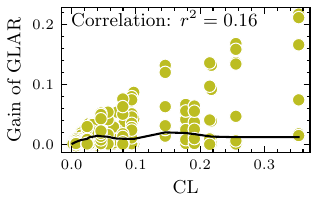}
    \caption{\textbf{Gain vs $\CL$.} Gain of each post-training method vs the calibration loss of the model.}
    \label{fig:app:gain:cl}
\end{figure}

\begin{figure}[ht]
    \centering
    \includegraphics[width=0.33\linewidth]{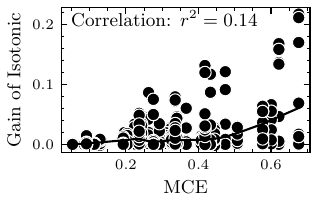}%
    \includegraphics[width=0.33\linewidth]{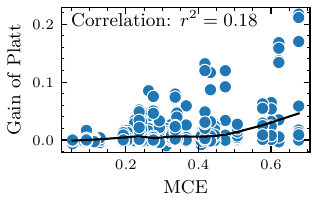}%
    \includegraphics[width=0.33\linewidth]{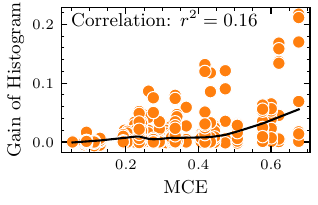}\\
    \includegraphics[width=0.33\linewidth]{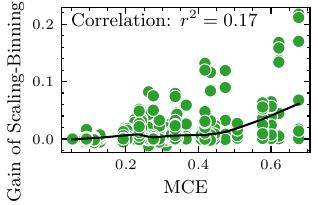}%
    \includegraphics[width=0.33\linewidth]{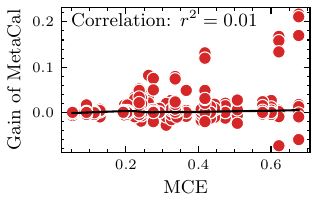}%
    \includegraphics[width=0.33\linewidth]{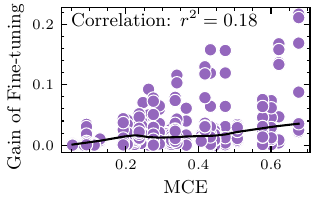}\\
    \includegraphics[width=0.33\linewidth]{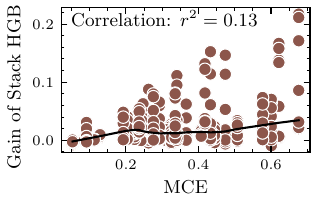}%
    \includegraphics[width=0.33\linewidth]{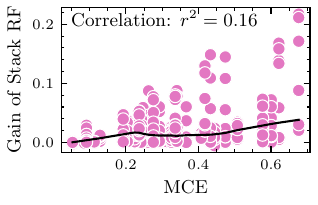}%
    \includegraphics[width=0.33\linewidth]{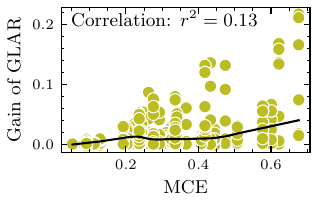}
    \caption{\textbf{Gain vs $\mce$.} Gain of each post-training method vs the $\mce$ of the model.}
    \label{fig:app:gain:mce}
\end{figure}

\begin{figure}[ht]
    \centering
    \includegraphics[width=0.33\linewidth]{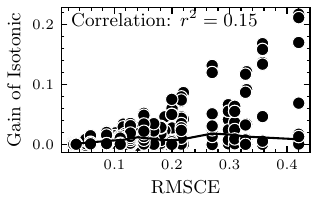}%
    \includegraphics[width=0.33\linewidth]{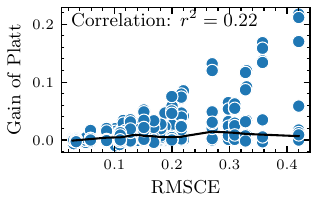}%
    \includegraphics[width=0.33\linewidth]{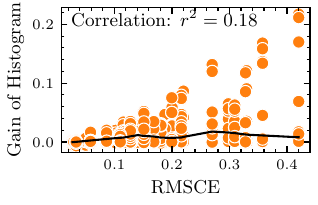}\\
    \includegraphics[width=0.33\linewidth]{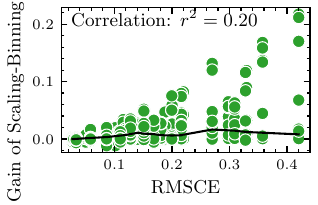}%
    \includegraphics[width=0.33\linewidth]{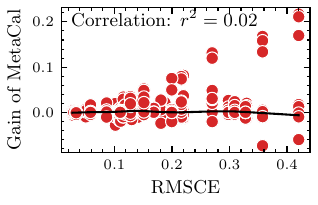}%
    \includegraphics[width=0.33\linewidth]{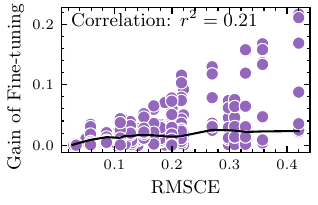}\\
    \includegraphics[width=0.33\linewidth]{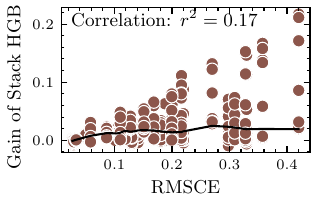}%
    \includegraphics[width=0.33\linewidth]{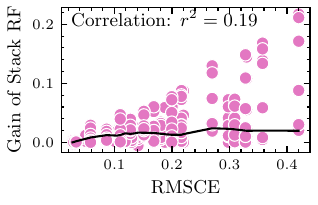}%
    \includegraphics[width=0.33\linewidth]{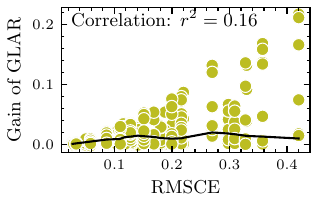}
    \caption{\textbf{Gain vs $\rmsce$.} Gain of each post-training method vs the $\rmsce$ of the model.}
    \label{fig:app:gain:rmsce}
\end{figure}

\clearpage

\section{Supplementary results}
\label{sec:supp-results}

\subsection{Influence of $t\sstar$: gains as a function of the utility regime}

We investigate the influence of the regime $t\sstar$ on the potential gain of post-training. We found that, on average across all models and datasets, values of $t\sstar$ in the range $[0.05, 0.95]$ yield higher regrets with a peak around $0.5$ (\autoref{fig:post:excess:t}). This corresponds to exchange rates of 1:19 to 1:1, with a peak around 1:1. Similarly to \citet{VanCalster2015calibration}, we observe that miscalibration is less likely to cause regret when $t\sstar$ is close to the event rate $\esp{Y}$ (\autoref{fig:recal:excess:event}).

\begin{figure}[ht]
    \centering
    \includegraphics{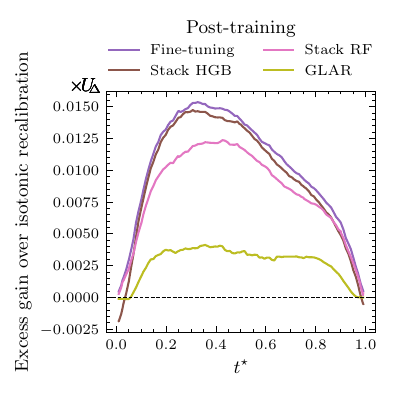}%
    \caption{Average gain of post-training over the gain of recalibration as a function of the utility-derived $t\sstar$.}
    \label{fig:post:excess:t}
\end{figure}
\begin{figure}[ht]
    \centering
    \includegraphics{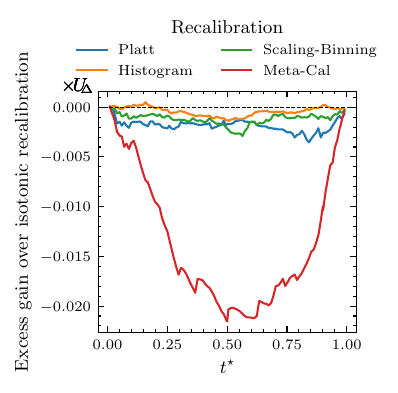}%
    \caption{Average gains of recalibration approaches compared to the gain of isotonic recalibration as a function of the utility-derived $t\sstar$.}
    \label{fig:recal:excess:t}
\end{figure}
\begin{figure}[ht]
    \centering
    \includegraphics{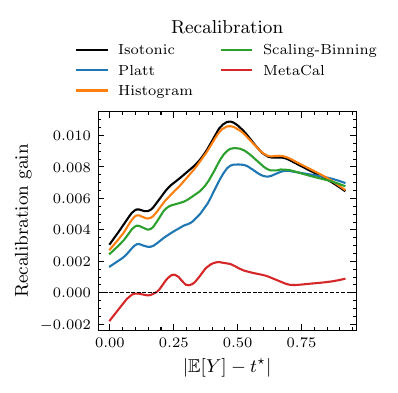}%
    \caption{Average gains of recalibration approaches compared to the gain of isotonic recalibration as a function of the distance of the utility-derived $t\sstar$ to the event rate $\esp{Y}$. LOWESS curves were fitted with a width of 0.2.}
    \label{fig:recal:excess:event}
\end{figure}

\end{document}

%% file: utils.tex
\usepackage[utf8]{inputenc} %
\usepackage[T1]{fontenc}    %
\def\mycitecolor{green!50!black}
\def\mylinkcolor{red!60!black}
\usepackage[colorlinks=true, linkcolor=\mylinkcolor, citecolor=\mycitecolor]{hyperref}       %
\usepackage{url}            %
\usepackage{booktabs}       %
\usepackage{amsfonts}       %
\usepackage{nicefrac}       %
\usepackage{microtype}      %
\usepackage{xcolor}         %
\usepackage{amsmath}
\usepackage{amssymb}
\usepackage{mathtools}
\usepackage{amsthm}
\usepackage{dsfont}
\usepackage{multirow}
\usepackage{longtable}
\usepackage{thmtools}
\usepackage{thm-restate}
\usepackage{enumitem}
\usepackage[capitalize,noabbrev]{cleveref}
\usepackage{tikz}
\usepackage{comment}
\usetikzlibrary{positioning, arrows.meta}
\usepackage{enumitem} %
\usepackage{subcaption}
\usepackage{mdframed}
\usepackage{tikz}
\usetikzlibrary{shapes.geometric, arrows}
\usetikzlibrary{arrows, chains, quotes, positioning,
                shapes.symbols}
\usepackage{dirtytalk}
\usepackage{float}
\usepackage{dblfloatfix}
\input{commands.tex}
\def\Q{F}
\def\S{H}

\def\Pc{\mathcal{P}}

\def\t{\ensuremath{t\sstar}}
\def\f{\ensuremath{f\sstar}}

\def\EU{\mathrm{EU}}

\usepackage{snaptodo}

\snaptodoset{block rise=2em}
\snaptodoset{margin block/.style={font=\tiny}}

\theoremstyle{plain}
\newtheorem{theorem}{Theorem}[section]
\newtheorem{proposition}[theorem]{Proposition}
\newtheorem{lemma}[theorem]{Lemma}
\newtheorem{corollary}[theorem]{Corollary}
\theoremstyle{definition}
\newtheorem{definition}[theorem]{Definition}

\theoremstyle{remark}

\newmdtheoremenv[
  backgroundcolor=lightgray!20, %
  linecolor=black, %
  innertopmargin=5pt, %
  innerbottommargin=0pt,
  skipabove=0pt, %
  skipbelow=0pt, %
  leftmargin=0,
  rightmargin=0,
  linewidth=0pt, %
  innerleftmargin=0pt, %
  innerrightmargin=0pt, %
  leftmargin=0pt, %
  rightmargin=0pt, %
]{colortheorem}{Theorem}
\newmdtheoremenv[
  backgroundcolor=lightgray!20, %
  linecolor=black, %
  innertopmargin=5pt, %
  innerbottommargin=0pt,
  skipabove=0pt, %
  skipbelow=0pt, %
  innerleftmargin=0pt, %
  innerrightmargin=0pt, %
  leftmargin=0pt, %
  rightmargin=0pt, %
  linewidth=0pt, %
]{colorproof}{Proof}
\def\eg{\emph{e.g.}}
\def\ie{\emph{i.e.}}

\usepackage{minitoc}

\mtcsettitle{minitoc}{Appendix Contents}

\title{Decision from Suboptimal Classifiers:\\Regret Pre- and Post-Calibration}

%% file: commands.tex
\makeatletter
\newcommand*{\centerfloat}{%
  \parindent \z@
  \leftskip \z@ \@plus 1fil \@minus \textwidth
  \rightskip\leftskip
  \parfillskip \z@skip}
\makeatother

\newcommand{\cm}{\mu^{-}}
\newcommand{\cp}{\mu^{+}}

\newcommand{\pr}[1]{\left( #1 \right)}

\newcommand{\brac}[1]{\left\{ #1 \right\}}
\newcommand{\cbr}[1]{\brac{#1}}
\newcommand{\bra}[1]{\left[ #1 \right]}

\newcommand{\R}{\mathbb{R}}

\newcommand{\EU}{\mathrm{EU}}

\newcommand{\Xc}{\mathcal{X}}

\newcommand{\esp}[2][]{\mathbb{E}_{\scriptscriptstyle#1}\!\bra{#2}}

\newcommand{\espk}[2]{\esp{#1\middle|#2}}

\newcommand{\var}[2][]{\mathbb{V}_{\!#1}\!\bra{#2}}

\newcommand{\vark}[3][]{\var[#1]{#2\,\middle|\,#3}}

\newcommand{\PP}[1]{\mathbb{P}\bra{#1}}
\newcommand{\PPG}[2]{\mathbb{P}\!\pr{#1\middle|#2}}
\newcommand{\1}[1]{\mathds{1}_{#1}}

\newcommand{\argmax}[1]{\underset{#1}{\operatorname{arg}\!\operatorname{max}}\;}

\newcommand{\sstar}{^{\star}}

\newcommand{\eg}{\textit{e.g}}
\newcommand{\ie}{\textit{i.e}\ }

\newcommand{\GL}{\mathrm{GL}}
\newcommand{\GLhat}{\widehat{\GL}}

\newcommand{\CL}{\mathrm{CL}}

\newcommand{\supp}[1]{\operatorname{supp}#1}

\newcommand{\deltastar}{\delta_{f\sstar\!\!, t\sstar}}
\newcommand{\rcl}{R^\CL}
\newcommand{\rclhat}{\smash{\hat{R}}^\CL}
\newcommand{\rgl}{R^\GL}
\newcommand{\lgl}{L^\GL}
\newcommand{\ugl}{U^\GL}
\newcommand{\lglhat}{\hat{L}^\GL}
\newcommand{\uglhat}{\hat{U}^\GL}
\newcommand{\rglh}{\rgl_\hh}

\newcommand{\rclht}{\rcl_{\hh\!,t}}
\newcommand{\rclhtstar}{\rcl_{\smash{\hh}\!,t\sstar}}
\newcommand{\rclhtstarhat}{\rclhat_{\smash{\hh}\!,t\sstar}}
\newcommand{\rclhthat}{\rclhat_{\hh\!,t}}
\newcommand{\rht}{R_{\hh\!,t}}
\newcommand{\rhthat}{\hat{R}_{\hh\!,t}}
\newcommand{\rhtstarhat}{\hat{R}_{\smash{\hh}\!,t\sstar}}

\newcommand{\rglhat}{\hat{R}^{\GL}}

\newcommand{\rglhhat}{\rglhat_{\smash{\hh}}}
\newcommand{\lglh}{\lgl_\hh}
\newcommand{\lglhhat}{\lglhat_\hh}
\newcommand{\uglh}{\ugl_\hh}
\newcommand{\uglhhat}{\uglhat_\hh}
\newcommand{\vmin}{V_{\!\mathrm{min}}}
\newcommand{\svmin}{v_{\mathrm{min}}}
\newcommand{\vmax}{V_{\!\mathrm{max}}}
\newcommand{\vfhp}{\vark{\hstar(X)\mkern-2mu}{\mkern-2mu \hh(X)\!=\!p}}

\newcommand{\unorm}{U_{\!\!\scriptscriptstyle\Delta}}

\newcommand{\ece}{\mathrm{ECE}}
\newcommand{\mce}{\mathrm{MCE}}
\newcommand{\rmsce}{\mathrm{RMSCE}}

\definecolor{C0}{HTML}{1f77b4}
\definecolor{C1}{HTML}{ff7f0e}
\definecolor{C2}{HTML}{2ca02c}
\definecolor{C3}{HTML}{d62728}
\definecolor{C4}{HTML}{9467bd}
\definecolor{C5}{HTML}{8c564b}
\definecolor{C6}{HTML}{e377c2}
\definecolor{C7}{HTML}{7f7f7f}
\definecolor{C8}{HTML}{bcbd22}
\definecolor{C9}{HTML}{17becf}

\colorlet{mplblue}{C0}
\colorlet{mplorange}{C1}
\colorlet{mplgreen}{C2}
\colorlet{mpldgreen}{mplgreen!75!black}
\colorlet{mplred}{C3}
\colorlet{mplpurple}{C4}
\colorlet{mpldpurple}{mplpurple!75!black}
\colorlet{mplbrown}{C5}
\colorlet{mplpink}{C6}
\colorlet{mplgray}{C7}
\colorlet{mplolive}{C8}
\colorlet{mplcyan}{C9}

  {
    \begin{mdframed}[
    backgroundcolor=lightgray!20,
    topline=false,
    rightline=false,
    leftline=false,
    bottomline=false,
    ]
    \begin{proof}}%
  {\end{proof}
\end{mdframed}
}
  {\begin{mdframed}[
    backgroundcolor=mplblue!5,
    topline=false,
    rightline=false,
    leftline=false,
    bottomline=false,
    ]\begin{question}}%
  {\end{question}\end{mdframed}}

\def\hh{f}
\def\hstar{\hh^{\star}}
\def\c{c}
\def\ch{\c}
\def\chat{\hat{c}}
\def\chhat{\chat}
\def\cch{c\circ\!\hh}
\def\chch{\ch\circ\!\hh}
\def\hhp{\hh_{\Pc}}
\newcommand{\autoreff}[2][]{\hyperref[#2]{\autoref*{#2}#1}}

%% file: img/fig_flowchart.tex
\definecolor{flowpro}{HTML}{eae4ee}
\definecolor{flowstart}{HTML}{f4dfd2}
\colorlet{flowdecision1}{mplgreen!20}
\colorlet{flowdecision2}{mplgray!20}
\colorlet{flowstart}{mplcyan!20}

\tikzstyle{startstop} = [rectangle, rounded corners,
minimum height=1cm,
text centered,
draw=black,
fill=flowstart,
text width=1.5cm,
align=center
]

\tikzstyle{process} = [rectangle,
minimum width=2cm,
minimum height=1cm,
text centered,
text width=3cm,
draw=black,
fill=flowpro
]
\tikzstyle{processS} = [rectangle,
minimum height=1cm,
text centered,
text width=1.8cm,
draw=black,
fill=flowpro
]

\tikzstyle{decision} = [signal,
signal to=east and west,
text centered,
draw=black,
fill=flowdecision1,
minimum height=1cm,
text width=2cm,
align=center]

\tikzstyle{decisionL} = [signal,
signal to=east and west,
text centered,
draw=black,
fill=flowdecision1,
minimum height=1cm,
text width=3.5cm,
align=center]
\tikzstyle{decisionL2} = [signal,
signal to=east and west,
text centered,
draw=black,
fill=flowdecision1,
minimum height=1cm,
text width=2.8cm,
align=center]
\tikzstyle{decisionL3} = [signal,
signal to=east and west,
text centered,
draw=black,
fill=flowdecision2,
minimum height=1cm,
text width=2.8cm,
align=center]

\tikzstyle{arrow} = [thick,->,>=stealth]

\node (start) [startstop] {Trained\\Model};
\node (dec1) [decision, right of=start, xshift=1cm] {Suboptimal?\\$\rht > 0$};
\node (dec2) [decisionL, right of=dec1, xshift=2.5cm] {Regret comes from miscalibration?\\$\rht \approx \rclht, \rglh \approx 0$};

\node (pro2a) [process, below of=dec2, yshift=0.2cm] {Select recalibration as post-training};
\node (pro2b) [process, right of=dec2, xshift=2.6cm] {Select advanced post-training};

\node (dec3) [decisionL2, right of=pro2a, xshift=2.6cm] {Cost of selected method worth the potential gain $\rht$?};
\node (pro3a) [process, below of=dec3, yshift=0.2cm] {Post-train with selected method};
\node (pro3b) [processS, right of=dec3, xshift=1.9cm] {Don't\\post-train};
\node (dec4) [decisionL3, above right of=pro3b, xshift=1.1cm, yshift=0.385cm] {Is the expected utility satisfying?};

\node (pro4b) [processS, below of=dec4, xshift=2.6cm, yshift=0.2cm] {Find better\\features};
\node (pro5) [startstop, below of=pro4b, yshift=0.2cm] {Start over};
\node (pro4a) [startstop, left of=pro4b, xshift=-0.6cm] {Done};

\draw [arrow] (start) -- (dec1);
\draw [arrow] (dec1) -- node[anchor=south] {yes} (dec2);
\draw [arrow] (dec2) -- node[anchor=west] {yes} (pro2a);
\draw [arrow] (dec2) -- node[anchor=south] {no} (pro2b);
\draw [arrow] (dec1) |- ++(0,0.9) -| node[anchor=north,pos=0.01] {no} (dec4.north);
\draw [arrow] (pro2a) -- (dec3);
\draw [arrow] (pro2b) -- (dec3);
\draw [arrow] (dec3) -- node[anchor=west] {yes/idk} (pro3a);
\draw [arrow] (dec3) -- node[anchor=south] {no} (pro3b);
\draw [arrow] (pro3b) |- (dec4);
\draw [arrow] (dec4) -- node[anchor=west] {yes/idk} (pro4a);
\draw [arrow] (dec4) -| node[anchor=west] {no} (pro4b);
\draw [arrow] (pro4b) -- (pro5);
\draw [arrow] (pro3a) -- node[anchor=south] {start over with post-trained model} (pro5);

%% file: img/decision/xp/test_common/translate_models/model_translation.tex
\begin{longtable}{lll}
\caption{Detailed description of all the models used for the hate speech detection experiment.} \label{tab:model_translation} \\
\toprule
 &  & Latent Layer \\
Model & HuggingFace &  \\
\midrule
\endfirsthead
\caption[]{Detailed description of all the models used for the hate speech detection experiment.} \\
\toprule
 &  & Latent Layer \\
Model & HuggingFace &  \\
\midrule
\endhead
\midrule
\multicolumn{3}{r}{Continued on next page} \\
\midrule
\endfoot
\bottomrule
\endlastfoot
CNERG Hatexplain & \href{https://huggingface.co/Hate-speech-CNERG/bert-base-uncased-hatexplain}{Hate-speech-CNERG/bert-base-uncased-hate...} & classifier \\
\cline{1-3}
CNERG en MuRIL & \href{https://huggingface.co/Hate-speech-CNERG/english-abusive-MuRIL}{Hate-speech-CNERG/english-abusive-MuRIL} & classifier \\
\cline{1-3}
CNERG en mono & \href{https://huggingface.co/Hate-speech-CNERG/dehatebert-mono-english}{Hate-speech-CNERG/dehatebert-mono-englis...} & classifier \\
\cline{1-3}
CNERG portuguese & \href{https://huggingface.co/Hate-speech-CNERG/dehatebert-mono-portugese}{Hate-speech-CNERG/dehatebert-mono-portug...} & classifier \\
\cline{1-3}
CNERG tamil & \href{https://huggingface.co/Hate-speech-CNERG/tamil-codemixed-abusive-MuRIL}{Hate-speech-CNERG/tamil-codemixed-abusiv...} & classifier \\
\cline{1-3}
FB Roberta & \href{https://huggingface.co/facebook/roberta-hate-speech-dynabench-r4-target}{facebook/roberta-hate-speech-dynabench-r...} & classifier.out\_proj \\
\cline{1-3}
\end{longtable}

%% file: img/decision/xp/test_common/translate_tables/ds_translation.tex
\begin{longtable}{llllll}
\caption{Detailed description of all the datasets used for the hate speech detection experiment.} \label{tab:ds_translation} \\
\toprule
 &  & Split & Input & Target & Pos. class \\
Dataset & HuggingFace or CSV &  &  &  &  \\
\midrule
\endfirsthead
\caption[]{Detailed description of all the datasets used for the hate speech detection experiment.} \\
\toprule
 &  & Split & Input & Target & Pos. class \\
Dataset & HuggingFace or CSV &  &  &  &  \\
\midrule
\endhead
\midrule
\multicolumn{6}{r}{Continued on next page} \\
\midrule
\endfoot
\bottomrule
\endlastfoot
Tweets & \href{https://huggingface.co/datasets/tweets\_hate\_speech\_detection}{tweets\_hate\_speech\_detect...} & train & tweet & label &  \\
\cline{1-6}
Speech18 & \href{https://huggingface.co/datasets/hate\_speech18}{hate\_speech18} & train & text & label &  \\
\cline{1-6}
Offensive & \href{https://huggingface.co/datasets/hate\_speech\_offensive}{hate\_speech\_offensive} & train & tweet & class & 0 \\
\cline{1-6}
Davidson & \href{https://huggingface.co/datasets/krishan-CSE/Davidson\_Hate\_Speech}{krishan-CSE/Davidson\_Hate...} & train+test & text & labels & 0 \\
\cline{1-6}
Gender & \href{https://huggingface.co/datasets/ctoraman/gender-hate-speech}{ctoraman/gender-hate-spee...} & train+test & Text & Label & 2 \\
\cline{1-6}
\multirow[t]{2}{*}{FRENK} & \href{https://huggingface.co/datasets/classla/FRENK-hate-en}{classla/FRENK-hate-en} & train+val+test & text & label &  \\
 & \href{https://huggingface.co/datasets/limjiayi/hateful\_memes\_expanded}{limjiayi/hateful\_memes\_ex...} & train+val+test & text & label &  \\
\cline{1-6}
Check & \href{https://huggingface.co/datasets/Paul/hatecheck}{Paul/hatecheck} & test & test\_case & label\_gold & hateful \\
\cline{1-6}
Tweets 2 & \href{https://huggingface.co/datasets/thefrankhsu/hate\_speech\_twitter}{thefrankhsu/hate\_speech\_t...} & train+test & tweet & label &  \\
\cline{1-6}
Open & \href{https://huggingface.co/datasets/parnoux/hate\_speech\_open\_data\_original\_class\_test\_set}{parnoux/hate\_speech\_open\_...} & test & tweet & class & 0 \\
\cline{1-6}
UCB & \href{https://huggingface.co/datasets/ucberkeley-dlab/measuring-hate-speech}{ucberkeley-dlab/measuring...} & train & text & hate\_speec... & y > 0.5 \\
\cline{1-6}
\multirow[t]{5}{*}{Merged} & \href{https://huggingface.co/datasets/tweets\_hate\_speech\_detection}{tweets\_hate\_speech\_detect...} & train & tweet & label &  \\
 & \href{https://huggingface.co/datasets/hate\_speech18}{hate\_speech18} & train & text & label &  \\
 & \href{https://huggingface.co/datasets/hate\_speech\_offensive}{hate\_speech\_offensive} & train & tweet & class & 0 \\
 & \href{https://huggingface.co/datasets/krishan-CSE/Davidson\_Hate\_Speech}{krishan-CSE/Davidson\_Hate...} & train+test & text & labels & 0 \\
 & \href{https://huggingface.co/datasets/ctoraman/gender-hate-speech}{ctoraman/gender-hate-spee...} & train+test & Text & Label & 2 \\
\cline{1-6}
\multirow[t]{5}{*}{Merged 2} & \href{https://huggingface.co/datasets/classla/FRENK-hate-en}{classla/FRENK-hate-en} & train+val+test & text & label &  \\
 & \href{https://huggingface.co/datasets/limjiayi/hateful\_memes\_expanded}{limjiayi/hateful\_memes\_ex...} & train+val+test & text & label &  \\
 & \href{https://huggingface.co/datasets/Paul/hatecheck}{Paul/hatecheck} & test & test\_case & label\_gold & hateful \\
 & \href{https://huggingface.co/datasets/thefrankhsu/hate\_speech\_twitter}{thefrankhsu/hate\_speech\_t...} & train+test & tweet & label &  \\
 & \href{https://huggingface.co/datasets/parnoux/hate\_speech\_open\_data\_original\_class\_test\_set}{parnoux/hate\_speech\_open\_...} & test & tweet & class & 0 \\
\cline{1-6}
DynGen & CSV: \href{https://raw.githubusercontent.com/bvidgen/Dynamically-Generated-Hate-Speech-Dataset/main/Dynamically%20Generated%20Hate%20Dataset%20v0.2.3.csv}{bvidgen/Dynamic...} &  & text & label & hate \\
\cline{1-6}
\multirow[t]{11}{*}{Merged3} & \href{https://huggingface.co/datasets/tweets\_hate\_speech\_detection}{tweets\_hate\_speech\_detect...} & train & tweet & label &  \\
 & \href{https://huggingface.co/datasets/hate\_speech18}{hate\_speech18} & train & text & label &  \\
 & \href{https://huggingface.co/datasets/hate\_speech\_offensive}{hate\_speech\_offensive} & train & tweet & class & 0 \\
 & \href{https://huggingface.co/datasets/krishan-CSE/Davidson\_Hate\_Speech}{krishan-CSE/Davidson\_Hate...} & train+test & text & labels & 0 \\
 & \href{https://huggingface.co/datasets/ctoraman/gender-hate-speech}{ctoraman/gender-hate-spee...} & train+test & Text & Label & 2 \\
 & \href{https://huggingface.co/datasets/classla/FRENK-hate-en}{classla/FRENK-hate-en} & train+val+test & text & label &  \\
 & \href{https://huggingface.co/datasets/limjiayi/hateful\_memes\_expanded}{limjiayi/hateful\_memes\_ex...} & train+val+test & text & label &  \\
 & \href{https://huggingface.co/datasets/Paul/hatecheck}{Paul/hatecheck} & test & test\_case & label\_gold & hateful \\
 & \href{https://huggingface.co/datasets/thefrankhsu/hate\_speech\_twitter}{thefrankhsu/hate\_speech\_t...} & train+test & tweet & label &  \\
 & \href{https://huggingface.co/datasets/parnoux/hate\_speech\_open\_data\_original\_class\_test\_set}{parnoux/hate\_speech\_open\_...} & test & tweet & class & 0 \\
 & CSV: \href{https://raw.githubusercontent.com/bvidgen/Dynamically-Generated-Hate-Speech-Dataset/main/Dynamically%20Generated%20Hate%20Dataset%20v0.2.3.csv}{bvidgen/Dynamic...} &  & text & label & hate \\
\cline{1-6}
\end{longtable}